\documentclass{article}

\usepackage{float}
\usepackage{microtype}
\usepackage{graphicx}
\usepackage{subcaption}
\usepackage{booktabs} 
\usepackage{tikz}
\usetikzlibrary{arrows.meta,positioning}
\usepackage{float}
\usepackage{hyperref}

\usepackage[preprint]{icml2026}

\usepackage{amsmath}
\usepackage{amssymb}
\usepackage{mathtools}
\usepackage{amsthm}

\usepackage[capitalize,noabbrev]{cleveref}

\theoremstyle{plain}
\newtheorem{theorem}{Theorem}[section]

\newtheorem{lemma}[theorem]{Lemma}
\newtheorem{corollary}[theorem]{Corollary}
\theoremstyle{definition}

\theoremstyle{remark}

\usepackage[textsize=tiny]{todonotes}

\icmltitlerunning{Towards Infinitely Long Neural Simulations: Self-Refining Neural Surrogate Models for Dynamical Systems}

\begin{document}

\twocolumn[
  \icmltitle{Towards Infinitely Long Neural Simulations: Self-Refining Neural Surrogate Models for Dynamical Systems }



  \icmlsetsymbol{equal}{*}

  \begin{icmlauthorlist}

    \icmlauthor{Qi Liu}{courant}
    \icmlauthor{Laure Zanna}{courant}
    \icmlauthor{Joan Bruna}{courant}
  \end{icmlauthorlist}

  \icmlaffiliation{courant}{Courant Institute School of Mathematics, Computing, and Data Science, New York University, USA}

  \icmlcorrespondingauthor{Qi Liu}{ql2221@nyu.edu}

  \icmlkeywords{Machine Learning, ICML}

  \vskip 0.3in
]



\printAffiliationsAndNotice{}  

\begin{abstract}
 Recent advances in autoregressive neural surrogate models have enabled orders-of-magnitude speedups in simulating dynamical systems. However, autoregressive models are generally prone to \textit{distribution drift}: compounding errors in autoregressive rollouts that severely degrade generation quality over long time horizons. Existing work attempts to address this issue by implicitly leveraging the inherent trade-off between short-time accuracy and long-time consistency through hyperparameter tuning. In this work, we introduce a unifying mathematical framework that makes this tradeoff explicit, formalizing and generalizing hyperparameter-based strategies in existing approaches. Within this framework, we propose a robust, hyperparameter-free model implemented as a conditional diffusion model that balances short-time fidelity with long-time consistency by construction. Our model, \textit{Self-refining Neural Surrogate model} (\textit{SNS}), can be implemented as a standalone model that refines its own autoregressive outputs or as a complementary model to existing neural surrogates to ensure long-time consistency. We also demonstrate the numerical feasibility of \textit{SNS} through high-fidelity simulations of complex dynamical systems over arbitrarily long time horizons.
\end{abstract}

\section{Introduction}
Understanding and simulating physical systems is a foundational aspect of many scientific and engineering fields, including earth system modeling, neural science, and robotics \citep{lorenz1963,chariker2016a,ijspeert2013}. Dynamical systems are the major mathematical tools used to model these systems, which describe how the states of the system evolve over time through partial differential equations (PDEs).  Traditionally, these PDEs are solved numerically, which requires fine spatial discretizations to capture the local dynamics accurately. However, fine spatial discretization imposes constraints on the temporal discretization as well, such as the Courant-Friedrichs-Lewy (\textit{CFL}) condition for numerical stability, leading to unfeasible computational costs \cite{Lewy1928}.

Deep learning techniques have been incorporated in various ways to address this challenge. In particular, autoregressive neural surrogate models have been deployed to replace numerical schemes entirely \citep{li2020, stachenfeld2022learned, kochkov2024}. Autoregressive neural surrogates, compared to their numerical counterparts, can achieve orders of magnitude computational speedups by reducing the number of function evaluations and leveraging parallel processing on GPUs \cite{Kurth2023}. These neural surrogate models are being applied to weather forecasting \citep{pathak2022, bi2023} and climate modeling \citep{kochkov2024, subel2024, dheeshjith2025}. 

\begin{figure*}
    \centering
    \includegraphics[width = 0.99\linewidth]{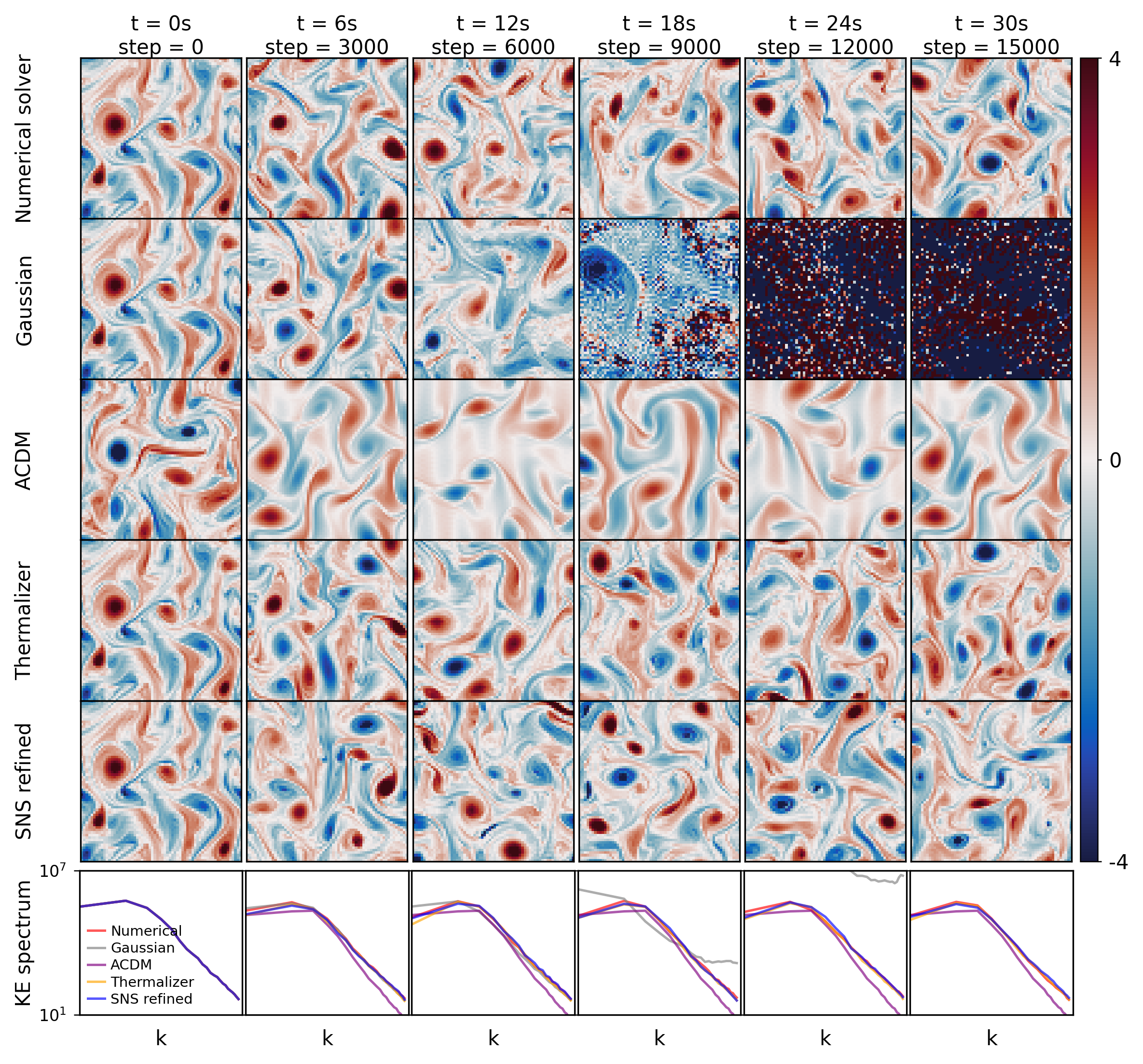}
    \caption{\textit{Top 5 rows:} Vorticity fields of the Kolmogorov flow from a numerical solver, Gaussian approximation of the transition density, \textit{ACDM} with 200 denoising steps, \textit{Thermalizer} and \textit{SNS} refined trajectories for Kolmogorov flow over a trajectory of $15,000$ steps. \textit{Bottom row:} kinetic energy spectra at each timestep, averaged over $20$ randomly initialized trajectories.}\label{fig:kol_traj}
\end{figure*} 

Autoregressive models take their own outputs as inputs. Despite the tremendous success of autoregressive neural surrogate models in achieving short-term accuracy, the performance tends to degenerate over longer time horizons \citep{Chattopadhyay2023,Bonavita2024,Parthipan2024,Bach2024}. Autoregressive models are prone to \textit{distribution drift}: they are trained on data from a distribution, but there is no guarantee that their outputs will remain within it during inference. This motivates decomposing the risk of an autoregressive model into \textit{conditional approximation error} and \textit{out-of-distribution} (\textit{OOD}) \textit{error}. Many existing efforts to enhance long rollouts' performance implicitly balance a tradeoff between these two errors. However, most prior work relies on heuristic strategies that yield strong empirical results but offer limited theoretical justification.

In this work, we will provide a general mathematical framework, based on \textit{multi-noise-level denoising oracle}, that explicitly accounts for the trade-off between \textit{conditional approximation error} and \textit{OOD error}. Estimators of the \textit{multi-noise-level denoising oracle} create a flexible design space for seeking an optimal balance between the two errors. We propose our model, \textit{Self-refining Neural Surrogate model} (\textit{SNS}), which seeks to achieve the optimal balance between the errors within this framework.  During inference, \textit{SNS} dynamically determines the noise level for the conditioning and the denoising strategies. We implement \textit{SNS} as a conditional diffusion model \citep{SohlDickstein2015, song2020} trained on a denoising diffusion probabilistic model (DDPM) style denoising score matching objective \cite{Ho2020}. \textit{SNS} can be used as a standalone model that refines its own output or as a complementary model alongside existing neural surrogates to enhance long-term performance. We demonstrate the numerical feasibility of \textit{SNS} through high fidelity simulations over a long time horizon.  

\subsection*{Main contributions}
We propose the \textit{multi-noise-level denoising oracle} with an explicit representation of the approximation-\textit{OOD} tradeoff to transform a heuristic goal into a concrete, learnable objective. We present our \textit{self-refining neural surrogate} model that seeks to dynamically balance the tradeoff within this framework.

\section{Related Work}
\label{sec2}

\subsection*{Diffusion models as surrogates}

Conditional diffusion models have been utilized by \textit{GenCast} \cite{kohlBenchmarkingAutoregressiveConditional2024} and autoregressive conditional diffusion model (\textit{ACDM})   \cite{price2023gencast} for probabilistic simulations of dynamical systems. \textit{ACDM} follows the \textit{Conditional Diffusive Estimator}(CDiffE) framework \cite{batzolis2021conditional}, in which the conditioning variables and target outputs are jointly diffused. In contrast, \textit{GenCast} conditions on clean and unperturbed inputs, following prior work on conditional diffusion models \citep{song2019generative, karras2022elucidating}. Our approach, \textit{SNS} , is also implemented as a conditional diffusion model; however, a key distinction is that both \textit{ACDM} and \textit{GenCast} condition their generation on previous autoregressive outputs under the assumption that these inputs remain in distribution. As a result, both models remain susceptible to \textit{distribution drift} arising from the accumulation of out-of-distribution (OOD) errors during rollout.

A well-known limitation of diffusion-based models (including ours) is their long inference time, which arises from degradation in sample quality when the reverse process is insufficiently discretized. Consequently, substantial effort has been devoted to accelerate diffusion inference, including consistency models \cite{song2023consist} and progressive distillation techniques \cite{salimans2022progressivedistillationfastsampling}. To the best of our knowledge, these approaches have not yet been applied to autoregressive conditional diffusion models for faster simulation. Several recent works propose alternative strategies to reduce computational cost. \cite{shehata2025improved} accelerates inference by skipping portions, especially the end, of the reverse process via Tweedie's formula. \textit{Dyffusion} \cite{cachay2023dyffusion} avoids solving the reverse stochastic differential equation by replacing the forward and backward processes with a learned interpolator and forecaster. For high-dimensional spaces, \cite{gao2024} reduces inference cost by performing diffusion in a learned low-dimensional latent space. \textit{PDE-Refiner} \cite{lippe2023pderefiner}, one of the first diffusion-based methods for PDE modeling, mitigates computational overhead by acting as a complementary refinement module on top of a point-wise autoregressive predictor. Similar to \textit{PDE-Refiner}, our model, \textit{SNS}, can also work around the computational cost of diffusion-based approaches by dynamically truncating the initial stages of the reverse process with a rough next state estimate from a point-wise autoregressive predictor. 

\subsection*{Mitigating Distribution Drift}

Recent work has focused on mitigating \textit{distribution drift} in autoregressive models by exploiting properties of the stationary distribution \citep{jiang2023training, schiff2024dyslim, pedersen2025}. In particular, the \textit{Thermalizer} framework  \cite{pedersen2025}, which dynamically maps out-of-distribution samples back to the stationary distribution through a diffusion model, is conceptually similar to our model. However, \textit{Thermalizer}'s objective is solely to ensure long-time stability while we consider the optimal tradeoff one can make to ensure stability while respecting temporal dynamics. Prior work by \cite{stachenfeld2022learned} demonstrated that employing a non-degenerate Gaussian approximation to the transition density improves long-time consistency. Moreover, unrolling this Gaussian approximation and minimizing mean squared error (MSE) over multiple time steps has been shown to further enhance long-horizon performance \citep{lusch2018a, Um2020, vlachas2020}. Conditioning the model on a larger temporal context has also proven effective in reducing \textit{distribution drift} \cite{nathaniel2026, zhang2025frame}. While our approach offers an alternative mechanism for extending the consistency horizon of simulations, our primary objective is to provide theoretical insight into the conditions and mechanisms that enable truly infinite-length generation.

\section{Problem Setup}
\label{sec3}
\subsection*{Background}
Consider a dynamical system in $ \mathbb{R}^n $ with an unknown forward operator $\mathcal{F} : \mathbb{R}^n \rightarrow \mathbb{R}^n$ and initial condition $\mathbf{x_0}$, s.t.
\begin{equation*}
    \partial_t \mathbf{x}(t) = \mathcal{F}(\mathbf{x}(t),t), \quad \mathbf{x}(0) = \mathbf{x}_0.
\end{equation*}
For a typical time scale $\Delta t$, we denote $\{\mathbf{x}_t\}_{t\in\mathbb{N}}$ to be the time-discretized snapshots of the system where $\mathbf{x}_t = \mathbf{x}(t ~\Delta t)$. Given samples of such snapshots as training data, the goal is to generate dynamically consistent trajectories given new initial conditions. 

The operator $\mathcal{F}$ is, in general, nonlinear and exhibits chaotic dynamics, i.e., small perturbations in initial states evolve to states that are considerably different \cite{lorenz1963}. To account for the sensitivity to uncertainty in initial conditions and numerical discretization errors, it is more appropriate to adopt a probabilistic framework of the system. 
We can view $\{\mathbf{x}_t\}_{t\in\mathbb{N}}$ as a discrete realization of a continuous time Markovian stochastic process, $\{\mathbf{X}_\tau\}_{\tau\in \mathbb{R}}$, i.e. $\mathbf{x}_t = \mathbf{X}_{t \Delta \tau}(\omega)$. We denote the path density of the process $p_t := p(\mathbf{x}_0,\dots,\mathbf{x}_t)$ which satisfies the recurrence relation: $p_{t+1} = p(\mathbf{x}_{t+1}, (t+1)\Delta \tau \mid \mathbf{x}_t, t\Delta \tau) ~  p_t$, where the transition density $p(\mathbf{x}_{t+1}, (t+1)\Delta \tau \mid \mathbf{x}_t, t\Delta \tau)$ is the conditional distribution of $\mathbf{X}_{(t+1)\Delta \tau}$ given $\mathbf{X}_{t \Delta \tau} = \mathbf{x}_{t}$. Here we assume that the Markov transition kernel admits a density, i.e. $\mathcal{Q}
(\mathbf{x}_{t},d(\mathbf{x}_{t+1})) = p(\mathbf{x}_{t+1} , (t+1)\Delta \tau\mid \mathbf{x}_{t}, t \Delta \tau)\,d\mathbf{x}_{t+1}$.  In addition, if the system is autonomous and ergodic, the transition density reduces to $p(\mathbf{x}_{t+1}\mid \mathbf{x}_t)$ and a unique limiting distribution exists which coincides with the stationary distribution. We denote this distribution as $\mu := \mu_\infty =  \lim_{t\rightarrow\infty} \mu_t$, where $\mu_t$ is the marginal distribution of $\mathbf{x}_t$, i.e. $\mu_t= \mathcal{Q}^t_\#(\mathbf{x}_{0},\mathbf{x}_{t})p(\mathbf{x}_{0})$. We choose to focus on ergodic systems in this work, but the method can be extended to time-dependent systems as well. Because of the Markovian properties, it suffices to learn the one-step conditional distribution and generate new trajectories by sampling from the learned distribution, $\hat{p}$, autoregressively as follows: given $\mathbf{x}_0$, sample $\hat{\mathbf{x}}_{t+1} \sim \hat{p}(\mathbf{x}_{t+1} \mid \hat{\mathbf{x}}_t)$ recursively with $\hat{\mathbf{x}}_0 = \mathbf{x}_0$.

\begin{figure}
    \centering
    \includegraphics[width=\linewidth]{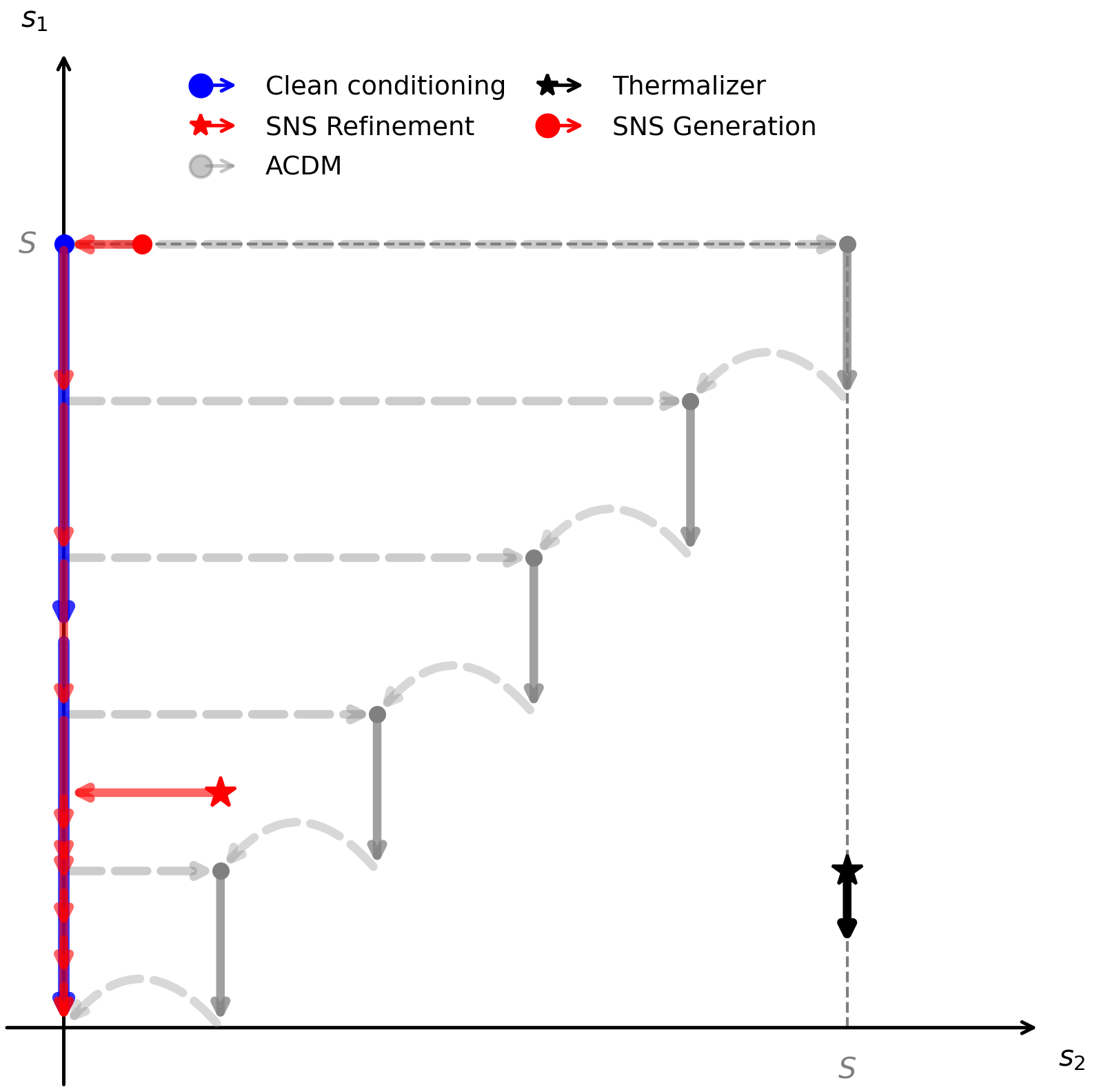}
    \caption{\textit{Phase space for the forward and reverse process} : Moving to the right and moving to the left corresponds to the forward process and the reverse process for $\mathbf{x}_{t-1}$. Moving up and down corresponds to the forward and the reverse process for $\mathbf{x}_t$, respectively. The dashed curved arrows represent a jump in the reverse process.}
    \label{fig:phase_space}
\end{figure}

\subsection*{Modeling transition density via diffusion models}
The transition density has the property that $p(\mathbf{x}_{t+1}\mid \mathbf{x}_{t}) = \delta(\mathbf{x}_{t+1} - \mathbf{x}_t)$ as $\Delta \tau \rightarrow 0$, i.e. the transition density is a Dirac mass at the current snapshot without time evolution. So in the asymptotics of $\Delta \tau \ll 1$, it is reasonable to approximate the transition density with a ``sharp'' gaussian distribution: $p(\mathbf{x}_{t+1}\mid \mathbf{x}_{t}) = \mathcal{N}(\mathbf{x}_{t+1}; \mathcal{F}(\mathbf{x}_t),\Sigma(\Delta \tau) )$ where $\Sigma(\Delta \tau)$ has vanishing spectral radius as $\Delta \tau \rightarrow 0$. At a fixed time discretization $\Delta \tau$, both the mean and the covariance matrix can be learned, and we denote the learned Gaussian kernel: $\hat{p} = (\mathbf{x}_{t+1}\mid \mathbf{x}_{t}) = \mathcal{N}(\mathbf{x}_{t+1}; \mathcal{F}_\theta(\mathbf{x}_t),\Sigma_\theta)$. We refer to this approach as the Gaussian approximation and provide more details about the learning objective in \ref{Training}. 

Conditional diffusion models can also be used to generate samples from the transition density $p(\mathbf{x}_t | \mathbf{x}_{t-1})$. Diffusion models \citep{SohlDickstein2015, Ho2020, song2020}, when first proposed, enable transport between a source distribution $p^S$, usually Gaussian, and an unconditional target distribution $p^0$ by a forward process and a learned reverse process. We will use $s$ for diffusion time as a superscript, which is independent of the physical time $t$. Given $\mathbf{x}^0 \sim p^0$, the forward process:
\begin{equation*}
    d\mathbf{x}^s = \mathbf{f}(\mathbf{x}^s,s)ds + g(s)d\mathbf{w}^s, \quad s \in [0,S]
\end{equation*}
is a stochastic process that starts at the target distribution and ends at the source distribution, i.e. $\mathbf{x}^0 \sim p^0, \mathbf{x}^S \sim p^S$. One could generate new samples from the target distribution by solving the SDE associated with the reverse process backwards in time (from $s=S$ to $s=0$) with initial condition $\mathbf{x}^S \sim p^1$\cite{anderson1982reverse}: 
\begin{equation*}
    d\mathbf{x}^s = [\mathbf{f}(\mathbf{x^s},s)- g(s)^2 \nabla \log p^s(\mathbf{x}^s)]ds + g(s)d\bar{\mathbf{w}}^s
\end{equation*}
where $p^s(\mathbf{x}^s)$ is the marginal density of $\mathbf{x}^s$. Solving the reverse SDE requires estimating the score, $\mathbf{s}_\theta \approx \nabla \log p^s(\mathbf{x}^s)$, which can be done via various training objectives. In particular, \cite{Ho2020} proposed a denoising objective. Defining the denoising oracle as:
\begin{equation}
D(\mathbf{x}^s, s) := \mathbb{E}[\mathbf{x}\mid \mathbf{x}^s]
\label{eq1}
\end{equation}
Tweedie's formula \cite{robbins1992empirical}, relates the denoising oracle to the score explicitly via: $\nabla \log p^s(\mathbf{x}^s) = -\sigma_s^{-2}(\mathbf{x}^s - D(\mathbf{x}^s,s))$, where $\sigma_s^2 = \int_{0}^s g(s')^2ds'$.

Conditional diffusion models enable sampling from a conditional distribution,  $p(\mathbf{x} \mid \mathbf{c})$. One could consider evolving the conditional reverse process SDE:
\begin{equation*}
    d\mathbf{x}^s = [\mathbf{f}(\mathbf{x^s},s)- g(s)^2 \nabla \log p^s(\mathbf{x}^s\mid \mathbf{c})]ds + g(s)d\bar{\mathbf{w}}^s
\end{equation*}
to generate samples from $p(\mathbf{x} \mid \mathbf{c})$. Multiple methods to estimate the conditional score $\nabla \log p^s(\mathbf{x}^s\mid \mathbf{c})$ have been proposed, such as via Bayes's rule \cite{dhariwal2021}, classifier-free guidance \cite{ho2022classifier}, and conditional score matching. In particular, \cite{batzolis2021conditional} considered the loss function,
\begin{equation*}
\begin{aligned}
\mathcal{L}(\theta)
= \frac{1}{2}\,
&\mathbb{E}_{s \sim \mathcal{U}(0,S)}
\mathbb{E}_{\mathbf{x}_0,\,\mathbf{x}_s \sim p_s(\mathbf{x}_s \mid \mathbf{x}_0)} \\
&\left[
    \lambda(t)\,
    \left\|
        \nabla \log p^s(\mathbf{x}^s \mid \mathbf{x}^0)
        - \mathbf{s}_\theta(\mathbf{x}_s, \mathbf{c}, s)
    \right\|_2^2
\right].
\end{aligned}
\end{equation*}
and showed that the minimizer of the above loss function coincides with the minimizer of the loss function with the score $\nabla \log p^s(\mathbf{x}^s \mid \mathbf{x}^0)$ replaced by the conditional score: $\nabla \log p^s(\mathbf{x}^s \mid \mathbf{c})$. Further more, one could consider learning the score of the annealed conditional distribution $\nabla \log p^s(\mathbf{x}^s\mid \mathbf{c}^s)$, which coincides with the score of the joint distribution $p^s(\mathbf{x}^s,\mathbf{c}^s)$, and use it as an approximation to $\nabla \log p^s(\mathbf{x}^s \mid \mathbf{c})$ as we will discuss in \cref{sec4.1}.

\subsection*{Distribution drift in autoregressive models}

To make the distinction between short-horizon prediction accuracy and long-horizon physical consistency precise, it is useful to consider a finite-dimensional Markov system. Let $M$ denote the true transition matrix and $\hat{M}$ its learned approximation, with stationary distributions $\mu$ and $\hat{\mu}$, i.e. the Perron (left) eigenvectors of $M$, and $\hat{M}$, respectively. This naturally leads to two distinct notions of error.

First, the \textit{conditional approximation error} measures how well one-step transitions are approximated under the true stationary distribution:
\begin{equation*}
\mathcal{E}_{\text{cond}}
=
\sum_i \mu_i \,
\mathcal{D}(M_i \,\|\, \hat{M}_i),
\end{equation*}
where $M_i$ denotes the $i$-th row of $M$ and $\mathcal{D}$ is a divergence such as KL or total variation. And $\mu_i$ is the i-th component of $\mu$.
Second, the \textit{out-of-distribution (OOD) error} measures the discrepancy between the stationary distributions induced by the true and learned transition density, with the same or possibly different divergence, $\tilde{D}$:
\begin{equation*}
\mathcal{E}_{\text{uncond}}
=
\tilde{\mathcal{D}}(\mu \,\|\, \hat{\mu}).
\end{equation*}

While the \textit{conditional approximation error} captures one-step accuracy with in-distribution conditioning, the unconditional prediction error governs long-horizon rollout behavior. Noting that the \textit{conditional approximation error} is a metric on conditional distributions and the \textit{OOD} error is a metric on marginal distributions, any candidate approximation $\hat{M}$ immediately faces an inherent tradeoff between accurately reconstructing short-time dynamics and maintaining long-time consistency. The minimizer of the \textit{conditional approximation error}, by definition, has to condition on an unperturbed conditioning since any corruption of the conditioning decreases the mutual information between the two variables (Appendix \ref{AppC}). Thus, conditional generation on the exact previous state enables accurate approximation of the one-step transition density $p(\mathbf{x}_t \mid \mathbf{x}_{t-1})$. However, even
small errors in the conditioning state accumulate during rollouts, leading to
\textit{distribution drift}.  This mechanism underlies the linear growth of pathwise KL divergence over time. \cite{pedersen2025} Conversely, weakening the dependence on the conditioning state improves robustness by biasing generation towards the marginal or stationary distribution, but at the cost of accuracy in modeling the temporal dynamics.

This tradeoff can be formalized by interpolating between conditional and marginal
distributions. Introducing noise into the conditioning variable yields a family of
intermediate conditional distributions $\{p^{s_2}(\mathbf{x}_t \mid \mathbf{x}_{t-1}^{s_2})\}_{s_2\in[0,S]}$, which
recover the exact transition density in the limit $s_2 \to 0$ and marginal density $p(\mathbf{x}_t)$ as
$s_2 \to S$. Therefore, the choice of noise level directly controls the tradeoff
between temporal fidelity and long-term stability. 

Most existing approaches implicitly select a fixed operating point along this tradeoff
curve through hyperparameters or architectural choices. In contrast, our approach
explicitly models the full phase space via a \textit{multi-noise-level denoising oracle},
enabling adaptive selection of the appropriate balance during inference.

\section{Methodology}
\subsection{Multi-noise-level denoising oracle}
\label{sec4.1}
Denoting $\mathbf{y}_t := (\mathbf{x}_t,\mathbf{x}_{t-1})$, we define the \textit{multi-noise-level denoising oracle} to be the conditional expectation: 
\begin{equation}
    \mathbf{D}(\mathbf{x}^{s_1}_t, \mathbf{x}^{s_2}_{t-1}) := \mathbb{E}_{\mathbf{x}_t,\mathbf{x}_{t-1} \sim p(\mathbf{x}_t,\mathbf{x}_{t-1})}[\mathbf{y}_t \mid \mathbf{x}^{s_1}_t, \mathbf{x}_{t-1}^{s_2}] 
\end{equation}
where $p(\mathbf{x}_t,\mathbf{x}_{t-1})$ is the joint distribution of $(\mathbf{x}_t,\mathbf{x}_{t-1})$ and $\mathbf{x}_t^{s_1}$, $\mathbf{x}_{t-1}^{s_2}$ are independent realizations of the forward process: 
\begin{equation*}
    d\mathbf{x}^s = \mathbf{f}(\mathbf{x}^s,s)ds + g(s)d\mathbf{w}^s, \quad s \in [0,s_i]\quad  i= 1,2
\end{equation*}
with initial conditions $\mathbf{x}_t^0 = \mathbf{x}_t$, $\mathbf{x}_{t-1}^0 = \mathbf{x}_{t-1}$ respectively. It is important to note that although the forward processes are independent, $\mathbf{x}_t^{s_1}$ and $\mathbf{x}_{t_1}^{s_2}$ are not independent since the initial conditions are drawn from their joint distribution. This is the fundamental mathematical object underlying our method that allows a tradeoff between \textit{conditional approximation error} and \textit{OOD error}. The motivation for the \textit{multi-noise-level denoising oracle} is analogous to the tradeoff between conditional and marginal distribution discussed in \cref{sec3}. In the limit of $s_2 \rightarrow S$, the denoising oracle reduces to the unconditional denoising oracle as defined in equation \eqref{eq1}. In the limit of $s_2 \rightarrow 0$, the denoising oracle can be interpreted as a conditional denoising oracle which corresponds to the expected value of the transition density. The interpolation of these limits is exactly the tradeoff between \textit{approximation} error and \textit{OOD error}. \cref{fig3} demonstrates the tradeoff through an estimator of the \textit{multi-noise-level denoising oracle}.

\begin{figure}
    \centering
    \includegraphics[width=0.8\linewidth]{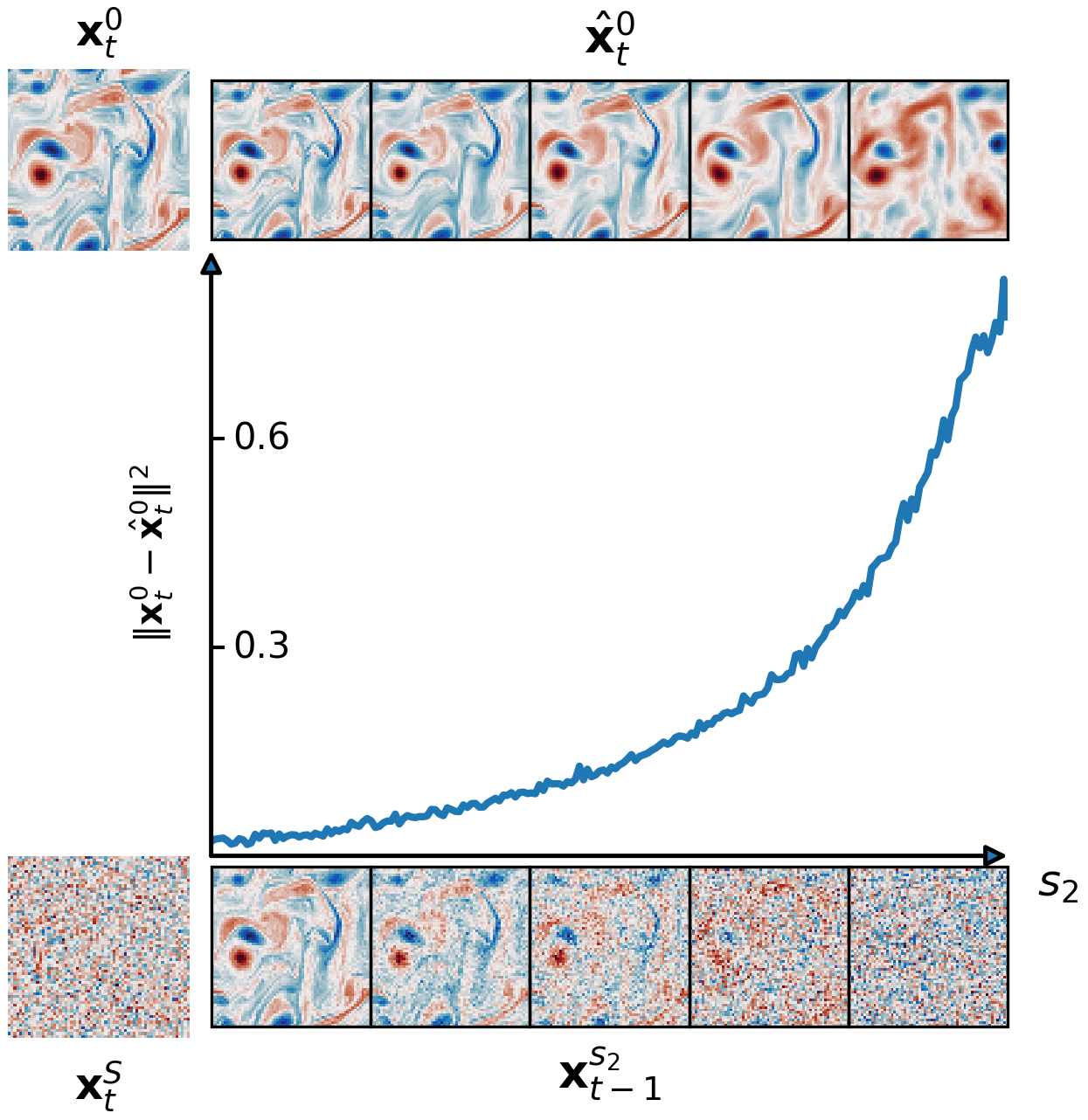}
    \caption{The x-axis corresponds to the diffusion time of the forward process for the conditional variable. The images below the x-axis are realizations of the forward process with the same initial condition $\mathbf{x}_{t-1}^{0}$ at different times $s$. The images in the top row are in direct correspondence with the images below via an estimator of the \textit{multi-noise-level denoising oracle} $\hat{\mathbf{x}}_t^0 =\mathbf{D}_{\theta}(\mathbf{x}_t^S, \mathbf{x}_{t-1}^{s})$. The y-axis in the $L^2$ distance between the denoised image and the original image. The degradation in denoising accuracy and the smoothing of fields with higher noise injection in the conditional variable align with the theory.}
    \label{fig3}
\end{figure}

Given the \textit{multi-noise-level denoising oracle} and initial condition $\hat{\mathbf{x}}_0 = \mathbf{x}_0$, one could use it directly for generating trajectories by following the recurrence relation $(\hat{\mathbf{x}}_t, \hat{\mathbf{x}}_{t-1}) = \mathbf{D}(\mathbf{z}_t, \mathbf{x}_{t-1}^0)$ where $\mathbf{z}_t \sim p^1$. However, this generates a trajectory based on point-wise estimates, and imperfect approximations of the oracle are still subject to \textit{OOD error} as a degenerate version of the Gaussian approximation of the transition density discussed in \cref{sec3}.

In order to enable probabilistic modeling, we define the \textit{multi-noise-level score}, $\nabla \log p^\mathbf{s}(\mathbf{y}^{\mathbf{s}}_t)$, to be:
\begin{equation*}
     (\nabla_{\mathbf{x}_t} \log p^{\mathbf{s}}(\mathbf{y}^{\mathbf{s}}_t),\nabla_{\mathbf{x}_{t-1}} \log p^\mathbf{s}(\mathbf{y}^{\mathbf{s}}_t))
\end{equation*}
where $\mathbf{s} = (s_1, s_2)$, $\mathbf{y}^\mathbf{s}_t = (\mathbf{x}_t^{s_1}, \mathbf{x}_{t-1}^{s_2})$ and $p^{\mathbf{s}}(\mathbf{y}^{\mathbf{s}}_t)$ is the marginal distribution of $(\mathbf{x}_t^{s_1}, \mathbf{x}_{t-1}^{s_2})$ in the forward process.
Analogous to the univariate case, we can relate the \textit{multi-noise-level denoising oracle} to the \textit{multi-noise-level score} via:
\begin{equation*}
    \nabla \log p^\mathbf{s}(\mathbf{y}^{\mathbf{s}}_t) = -\mathbf{\Sigma}_\mathbf{s}^{-1}( \mathbf{y}_t -  \mathbf{D}(\mathbf{x}^{s_1}_t, \mathbf{x}^{s_2}_{t-1}))
\end{equation*}
where $\mathbf{\Sigma}_\mathbf{s}$ is a diagonal matrix with entries $\mathbf{\Sigma_s}_{ii} = \int _0^{s_i} g^2(s')ds', i =1,2.$
Then, one could consider the coupled reverse process SDEs:
\begin{equation*}
\begin{aligned}
d\mathbf{x}_{t}^{s_1} =&
\Big[\mathbf{f}(\mathbf{x}_t^{s_1
},s_1)
- g(s_1)^2 \nabla_{\mathbf{x}_t} \log p^{\mathbf{s}}(\mathbf{y}^{\mathbf{s}}_t)  \Big]\,ds_1 \\
&+ g(s_1)\,d\bar{\mathbf{w}}^{s_1} 
\\
d\mathbf{x}_{t-1}^{s_2} =&
\Big[\mathbf{f}(\mathbf{x}_{t-1}^{s_2
},s_2)
- g(s_2)^2 \nabla_{\mathbf{x}_{t-1}} \log p^{\mathbf{s}}(\mathbf{y}^{\mathbf{s}}_t)  \Big]\,ds_2 \\
&+ g(s_2)\,d\bar{\mathbf{w}}^{s_2}
\end{aligned}
\end{equation*}
The SDEs are coupled through the \textit{multi-noise-level score}, but it poses no constraints on the coupling of the diffusion time $s_1$ and $s_2$. The evolution of the reverse process SDEs can be best understood as traversing the two-dimensional phase space of diffusion times. Given any point $(\tilde{s}_1, \tilde{s}_2) \in \mathbb{R}^+\times \mathbb{R}^+$, one could start the reverse process SDEs at any $(\mathbf{x}_t^{s_1}, \mathbf{x}_{t-1}^{s_2}) \sim p^\mathbf{s}(\mathbf{x}_t^{s_1}, \mathbf{x}_{t-1}^{s_2})$ and end at any point $(s_1', s_2')$, s.t. $s_i' \leq \tilde{s}_i$, following monotonically decreasing curves in $s_1$ and $s_2$. Many existing diffusion-based neural surrogate models can be interpreted as different strategies of traversing the phase space. \cref{fig:phase_space} demonstrates the corresponding evolution of the SDE under \textit{ACDM} \cite{kohlBenchmarkingAutoregressiveConditional2024}, \textit{thermalizer} \cite{pedersen2025}, and \textit{GenCast} (clean conditioning) \cite{price2023gencast}. 

Estimating the \textit{multi-noise-level score} can be done by minimizing the denoising score matching objective:
\begin{equation}
\begin{aligned}
\mathcal{L}(\theta) := &\mathbb{E}_{s_1,s_2 \sim \mathcal{U}(0,S), \mathbf{y}_t^0,\,\mathbf{y}^{\mathbf{s}} \sim p^\mathbf{s}(\mathbf{y}^\mathbf{s} \mid \mathbf{y}_t^0)}\\
&\| \nabla \log p_\mathbf{s}(\mathbf{y}^{\mathbf{s}}_t\mid \mathbf{y}_t^0) - \mathbf{\Phi}_\theta(\mathbf{y}_{t}^{\mathbf{s}}, s_1, s_2) \|^2
\end{aligned}
\label{eq3}
\end{equation}

In Appendix \ref{AppC}, we show that the minimizer of \cref{eq3} is the \textit{multi-noise-level score} and it coincides with the minimizer of a similar objective where the \textit{multi-noise-level score} is replaced with:
\begin{equation*}
  (\nabla_{\mathbf{x}_t} \log p^{s_1}(\mathbf{x}_{t}^{s_1} \mid \mathbf{x}_{t-1}^{s_2} ), \nabla_{\mathbf{x}_{t-1}} \log p^{s_2}(\mathbf{x}_{t-1}^{s_2} \mid \mathbf{x}_{t}^{s_1} ))
\end{equation*}

and the expectation is taken over $(\mathbf{x}_t^{s_1}, \mathbf{x}_{t-1}^{s_2}) \sim p^\mathbf{s}(\mathbf{x}_t^{s_1}, \mathbf{x}_{t-1}^{s_2})$. This provides us with another perspective on the reverse process. We can consider $\nabla_{\mathbf{x}_t} \log p^{s_1}(\mathbf{x}_{t}^{s_1} \mid \mathbf{x}_{t-1}^{s_2}$ as an approximation to the score of the transition density, $ \nabla p(\mathbf{x}_t \mid \mathbf{x}_{t-1})$, and independently evolve 
\begin{equation*}
\begin{aligned}
d\mathbf{x}_t^{s_1} =&
\Big[\mathbf{f}(\mathbf{x}_t^{s_1},s_1)
- g(s_1)^2 \nabla_{\mathbf{x}_t} \log p^{s_1}(\mathbf{x}_{t}^{s_1} \mid \mathbf{x}_{t-1}^{s_2}) \Big]\,ds_1 \\
& + g(s_1)\,d\bar{\mathbf{w}}^{s_1}
\end{aligned}
\end{equation*}
backward in time by drawing realizations of $\mathbf{x}_{t-1}^{s_2}$ from the forward process. Analogous to the case for \textit{multi-noise-level} denoising oracle, this approximation of the score introduces a bias but provides a better guarantee in long-term performance. Once again confirming the interpretation of $s_2$ as the variable controlling the tradeoff. This naturally raises the following question: \emph{what is the optimal choice of
starting point and traversal strategy in the reverse phase space
$(s_1,s_2)$ that best balances conditional approximation error against out-of-distribution (OOD) error?} We address this question by
constructing a conditional diffusion model that estimates both the \textit{multi-noise-level score} and the noise level all at once.  

\subsection{\textit{Self-refining Neural Surrogate model}}
\begin{algorithm}[tb]
  \caption{\textit{SNS} for refinement }
  \label{alg:refine}
  \begin{algorithmic}
  \REQUIRE Initial state $\mathbf{x}_0$, trained \textit{SNS} $\mathbf{\Phi}_{\theta}$, trained surrogate model $\mathbf{\Psi}_\theta$
    \FOR{$t=1$ {\bfseries to} $T$}
      \STATE $\mathbf{x}_t = \Psi_\theta(\mathbf{x}_{t-1})$
      \STATE $\hat{s}_i = \mathbf{\Phi}^{(i)}(\mathbf{x}_t,\mathbf{x}_{t-1}),\quad i = 1,2$
      \STATE $\mathbf{z}_1,\mathbf{z}_2 \sim \mathcal{N}(\mathbf{0},\mathbf{I})$
      \STATE $\mathbf{x}_{t}^{\hat{s}_1} = \sqrt{\bar{\alpha}_{\hat{s}_1}}\mathbf{x}_{t} + \sqrt{1-\bar{\alpha}_{\hat{s}_1}}\mathbf{z}_1$
      \STATE $\mathbf{x}_{t-1}^{\hat{s}_2} = \sqrt{\bar{\alpha}_{\hat{s}_2}}\mathbf{x}_{t-1} + \sqrt{1-\bar{\alpha}_{\hat{s}_2}}\mathbf{z}_2$
      \STATE $\mathbf{x}_{t-1}^{s} \leftarrow \mathbf{x}_{t-1}^{\hat{s}_2}$
      \FOR{$s=\hat{s}_2$ {\bfseries  to} $0$}
        \STATE $\mathbf{z}\sim\mathcal{N}(\mathbf{0},\mathbf{I})$
        \STATE $\mathbf{x}_{t-1}^{s-1} =
        \frac{1}{\sqrt{\alpha_s}}\!\left(
          \mathbf{x}_{t-1}^{s}
          - \frac{1-\alpha_s}{\sqrt{1-\bar{\alpha}_s}}\;
          \mathbf{\Phi}^{(4)}_\theta(\mathbf{x}_t^{\hat{s}_1},\mathbf{x}_{t-1}^{s})
        \right) + \sqrt{\beta_s}\,\mathbf{z}$
      \ENDFOR
      \STATE $\mathbf{x}_{t-1} \leftarrow \mathbf{x}_{t-1}^{0}$

      \vspace{2pt}
      \STATE $\mathbf{x}_t^{s} \leftarrow \mathbf{x}_t^{\hat{s}_1}$
      \FOR{$s=\hat{s}_1$ {\bfseries down to} $0$}
        \STATE $\mathbf{z}\sim\mathcal{N}(\mathbf{0},\mathbf{I})$
        \STATE $\mathbf{x}_t^{s-1} =
        \frac{1}{\sqrt{\alpha_s}}\!\left(
          \mathbf{x}_t^{s}
          - \frac{1-\alpha_s}{\sqrt{1-\bar{\alpha}_s}}\;
          \mathbf{\Phi}^{(3)}_\theta(\mathbf{x}_t^{s},\mathbf{x}_{t-1})
        \right) + \sqrt{\beta_s}\,\mathbf{z}$
      \ENDFOR
      \STATE $\mathbf{x}_t \leftarrow \mathbf{x}_t^{0}$

    \ENDFOR
  \end{algorithmic}
\end{algorithm}

\begin{figure*}
    \centering
    \includegraphics[width = 0.99\linewidth]{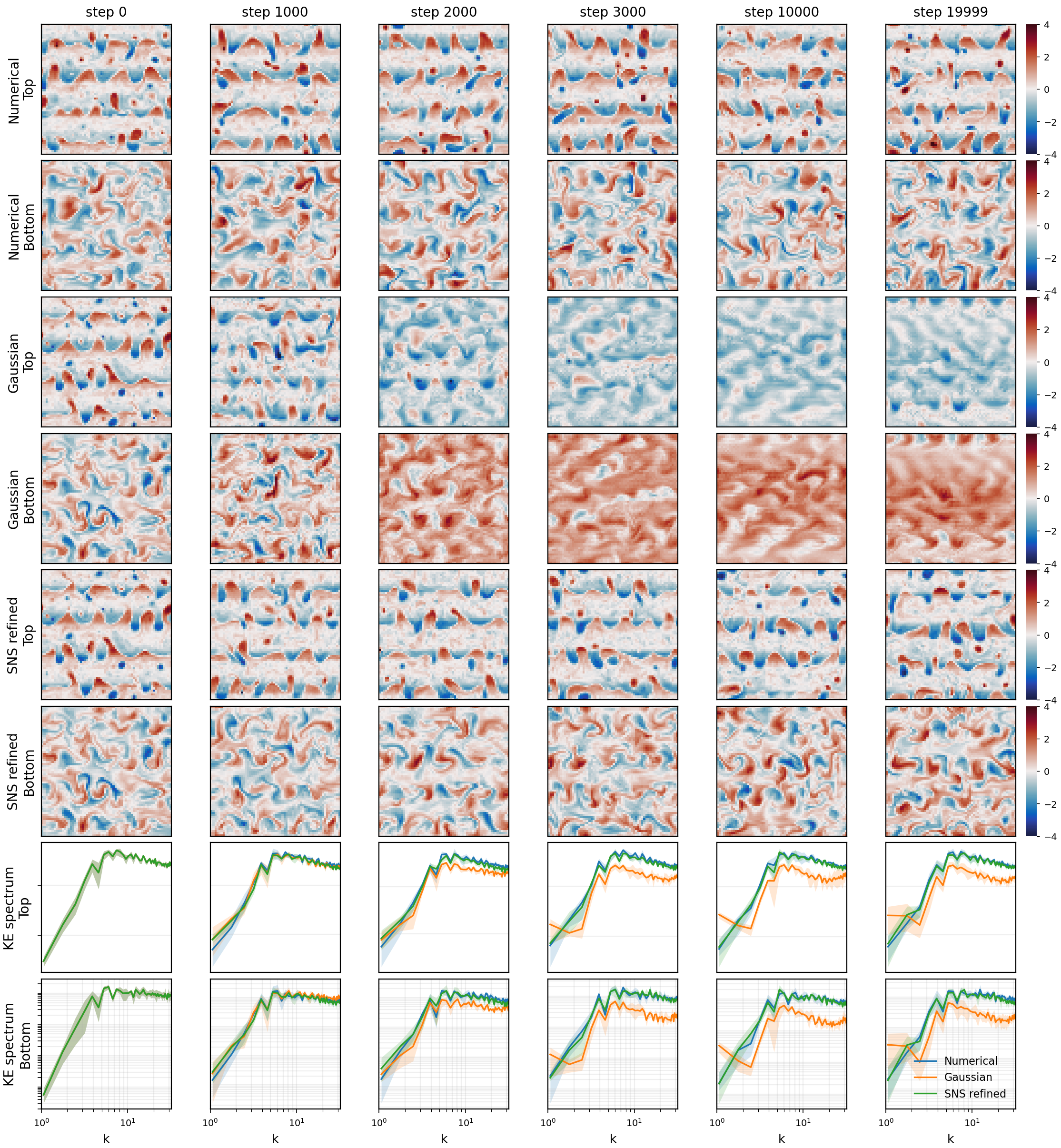}
    \caption{\textit{Top 6 rows:} Top and bottom layer vorticity fields of the two-layer QG system from a numerical solver, Gaussian approximation of the transition density, and \textit{SNS} refined trajectories for Kolmogorov flow over a trajectory of $20,000$ steps. \textit{Bottom 2 rows:} kinetic energy spectra of top and bottom layer at each timestep, averaged over $5$ randomly initialized trajectories, with min-max spread shaded. }\label{fig:QG_SNS_traj}
\end{figure*}
Given an initial condition $\hat{\mathbf{x}}_0$ and an estimator of the
\textit{multi-noise-level score}, one may generate autoregressive snapshots by
following any monotone path in diffusion-time space that starts at
$(S,\cdot)$ and terminates at $(0,\cdot)$. The remaining degree of freedom is
the evolution of the conditioning noise level $s_2$. In the idealized setting,
sampling from the true transition density $p(\mathbf{x}_t \mid \mathbf{x}_{t-1})$
corresponds to fixing $s_2 \equiv 0$ throughout the reverse process. However,
conditioning on an unperturbed autoregressive history is precisely what leads to
\textit{distribution drift} in practice, as confirmed by our numerical experiments and
previous work \cite{kohlBenchmarkingAutoregressiveConditional2024}.

To mitigate this effect, we do not assume the conditioning input to remain
in-distribution during rollouts. Instead, we dynamically determine an
appropriate starting point in the $(s_1,s_2)$ phase space and evolve the reverse
process toward the endpoint $(0,0)$. By leveraging the \textit{multi-noise-level score},
the model progressively refines both the current state and its conditioning,
effectively transporting the autoregressive input back toward the data manifold
during generation. This self-correcting mechanism motivates the name
\textit{Self-refining Neural Surrogate} (\textit{SNS}). The full algorithmic
implementation is provided in the appendix; below, we give a high-level overview
of the \textit{SNS} framework.
\begin{figure*}[t]
    \centering
    \includegraphics[width = \linewidth]{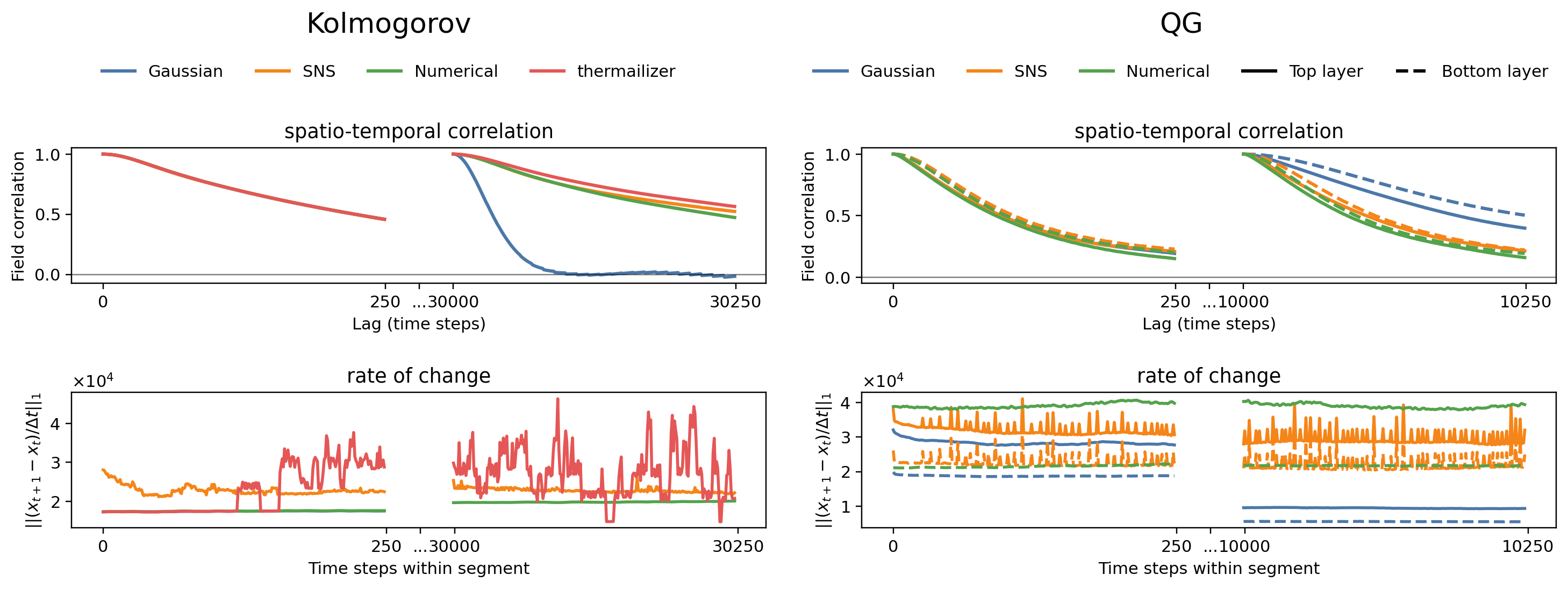}
    \caption{For a trajectory $\mathbf{x}_t \in \mathbb{R}^{d}$, we report two temporal consistency metrics. The spatio-temporal correlation at lag $\tau$ is computed as $C(\tau) = \left\langle \langle \mathbf{x}_t - \bar{\mathbf{x}}_t,\; \mathbf{x}_{t+\tau} - \bar{\mathbf{x}}_{t+\tau} \rangle\|\mathbf{x}_t - \bar{\mathbf{x}}_t\|_2^{-1} \, \|\mathbf{x}_{t+\tau} - \bar{\mathbf{x}}_{t+\tau}\|_2^{-1}\right\rangle_{t,\mathrm{IC}}$ where $\bar{\mathbf{x}}_t$ is the spatial mean of $\mathbf{x}_t$, and $\langle \cdot \rangle_{t,\mathrm{IC}}$ denotes averaging over time and initial conditions. The rate of change is computed from the discrete time derivative, $R(t)=\left\|\mathbf{x}_{t+1} - \mathbf{x}_t/{\Delta t}\right\|_1,$ and we plot the mean of $R(t)$ over different random initial conditions. The blue curves disappear in the bottom left plot for two different reasons. From 0 to 250, the blue curve is not visible because it overlaps with the green one, indicating strong temporal accuracy of the point-wise estimate in the short rollouts. From 30000 to 30250, the blue curve is missing because the point-wise estimates are returning NaN from the blow up in scale when the distribution drifted too far away from the stationary distribution.}\label{fig:temporal_metric}
\end{figure*} 

We estimate the \textit{multi-noise-level score} and the noise levels with a neural network, $\mathbf{\Phi}_\theta:\mathbb{R}^d \times \mathbb{R}^d \rightarrow \mathbb{R}^d \times \mathbb{R}^d \times \mathbb{R}^{2S}$. Denoting the outputs $\mathbf{\Phi}_\theta^{(3)}\in\mathbb{R}^d \times \mathbb{R}^d, (\mathbf{\Phi}_\theta^{(1)}, \mathbf{\Phi}_\theta^{(2)}) \in \mathbb{R}^{2S}$, the loss function to minimize is:
\begin{equation}
\begin{aligned}
   \mathcal{L}(\theta) &:= \mathbb{E}\{\| \nabla \log p_\mathbf{s}(\mathbf{y}^{\mathbf{s}}_t\mid \mathbf{y}_t^\mathbf{0}) - \mathbf{\Phi}^{(3)}_\theta(\mathbf{y}^{\mathbf{s}}_t) \|^2 \\
&+\lambda\sum_{s_1,s_2}[ \mathbf{1}_{s_1}\log \mathbf{\Phi}^{(1)}_\theta (\mathbf{y}^{\mathbf{s}}_t) + \mathbf{1}_{s_2}\log \mathbf{\Phi}^{(2)}_\theta (\mathbf{y}^{\mathbf{s}}_t)]\}
\label{loss}
\end{aligned}
\end{equation}
where $S$ is the number of discretization steps of the forward diffusion process and $\mathbf{1}_{s_i}$ are one-hot vectors encoding the true noise levels and the expectation is taken over $s_1, s_2 \sim U(\mathbb{N}[0,S]), \quad \mathbf{y}^{\mathbf{s}}_t, \mathbf{y}^{\mathbf{0}}_t \sim p^{\mathbf{s}}(\mathbf{y}^{\mathbf{s}}_t\mid \mathbf{y}^{\mathbf{0}}_t)$. We find that choosing the regularization weight, $\lambda$, to keep the two terms on a comparable scale works well.

Given the previous autoregressive output $\hat{\mathbf{x}}_{t-1}$, \textit{SNS} can be used directly for generation by evolving the reverse process SDE from $(S,\mathbf{\Phi}^{(2)}_\theta(\mathbf{d}_t))$ where $\mathbf{d}_t = (\mathbf{z}_t,\hat{\mathbf{x}}_{t-1}), \mathbf{z}_t \sim p^S$. Alternatively, one could also refine the autoregressive output by evolving the reverse coupled SDEs from $(\mathbf{\Phi}^{(1)}_\theta(\hat{\mathbf{y}}_{t}),\mathbf{\Phi}^{(2)}_\theta(\hat{\mathbf{y}}_{t}))$ where $\hat{\mathbf{y}}_{t} = (\hat{\mathbf{x}}_{t}, \hat{\mathbf{x}}_{t-1})$. Note that for refinement, the autoregressive output can be obtained by any other surrogate model. \cref{fig:phase_space} illustrates the phase-space traversal strategy implemented by \textit{SNS}. Through numerical experiments, we found that denoising first in the conditional variable and then denoising the current frame had the best numerical performance when compared to many other candidate traversal paths. (\cref{traverse}) We do not claim this path is optimal; rather, we suspect this advantage reflects a numerical bias in estimating the \textit{multi-noise-level score}, and that all monotonically decreasing paths should be equivalent. We provide a proof that shows the equivalence of monotonically decreasing paths in distribution in \cref{tra_proof}, and we leave a more detailed investigation for future work.

\section{Numerical Considerations}
We present the pseudocode for using \textit{SNS} as a complementary model to provide long-term performance guarantee for fast surrogates targeting minimal \textit{conditional approximation error} (\cref{alg:refine}). We follow DDPM-style training, discretizing the forward process according to the noise schedule $\{\alpha^s,\beta^s\}_{s\in \mathbb{N}([0,S])}$ of \cite{Ho2020}. The training process and additional numerical details of \textit{SNS} are provided in Appendix~\ref{Training}. 

\textit{SNS} represents our first attempt to identify an optimal tradeoff between \textit{Conditional approximation error} and \textit{OOD error} within the proposed framework, and in its current form it still faces several numerical challenges. First, although access to the multi-noise-level score allows the use of many existing diffusion-based methods, it also leads to a more difficult optimization problem. In particular, the \textit{SNS} training objective departs from the conventional formulation: traditional score estimators condition on both the noisy state and the noise level, whereas SNS operates solely on the noisy states. As a result, \textit{SNS} requires a finer discretization of diffusion time compared to other autoregressive conditional diffusion models during the reverse process to generate high-fidelity next-frame predictions. This substantially increases the already high computational cost of simulating the full reverse diffusion process. Consequently, all numerical results in this work rely on a refinement strategy applied to a pointwise-estimate surrogate model (e.g., a neural net that directly outputs in one pass). While this cost may be reduced in future work—potentially enabling diffusion-based models to scale to long time-horizon generation—we restrict our attention here to refinement tasks.
Defining an appropriate evaluation metric for long-time simulations remains an open problem. We measure the \textit{conditional approximation error} with temporal consistency metrics such as spatio-temporal autocorrelation and rate of change in state space as described under \cref{fig:temporal_metric}. In terms of \textit{OOD} error, while density-based metrics such as the Kullback–Leibler divergence and Wasserstein distance are natural candidates, estimating KL divergence in high-dimensional settings is computationally prohibitive. Consequently, we restrict our evaluation to qualitative diagnostics until a principled theoretical framework for this class of tasks becomes available. In particular, we will assess the long generation consistency through the spectrum of the energy in the state space provided at the bottom of all plots. 

\subsection{Kolmogorov flow}
 Simulating flows governed by the Navier--Stokes equations has long been a central challenge for neural surrogate models. We first consider Kolmogorov flow, a system on which \textit{Thermalizer} has previously demonstrated strong performance. We implement \cref{alg:refine} using the same pointwise surrogate model employed by \textit{Thermalizer}, a Unet trained with a multi-step MSE objective. (\cref{Training}) Such models usually achieve high fidelity short-time predictions but diverge from the stationary distribution in the long run.  As shown in \cref{fig:kol_traj}, \textit{SNS} is able to maintain its diverging trajectories within the training distribution over long-time horizons.

While \textit{ACDM} remains numerically stable, it produces fields with biased large-scale structures, leading to overly smooth solutions. It is important to note that both \textit{Thermalizer} and \textit{SNS} were trained with a noising schedule of 1000 diffusion steps, whereas $\textit{ACDM}$ was evaluated with only 200 steps, making a direct performance comparison inherently unfair. Nevertheless, the computational cost of using \textit{ACDM} for generation renders it impractical for long-time simulations. In contrast, the computational cost of \textit{SNS} and \textit{Thermalizer} is approximately two orders of magnitude lower than that of \textit{ACDM}.

\cref{fig:temporal_metric} demonstrate the superior performance of \textit{SNS} compared to \textit{thermalizer} in temporal consistency. This is expected by the theory since \textit{thermalizer} only aims to preserve marginal distributions during its projections, while \textit{SNS} also accounts for the transition density. Although \textit{Thermalizer} and \textit{SNS} exhibit comparable performance in long-time simulations of Kolmogorov flow in terms of \textit{OOD} error, we emphasize that \textit{SNS} requires no hyperparameter tuning, whereas \textit{Thermalizer} relies on tuning two diffusion hyperparameters, $s_{\text{init}}$ and $s_{\text{stop}}$, which define the start and end points of its diffusion process.

\subsection{Quasigeostrophic Turbulence}
We also consider two-layer quasigeostrophic (QG) flows as a test case. This dynamical system is of central importance in oceanic and atmospheric sciences, where it serves as a reduced-order model for a wide range of geophysical phenomena \cite{QiDi}. We evaluate all models on the jet configuration, following the experimental setup of \cite{ross2023}.

All other neural surrogates we considered failed when trained using their respective published procedures. In particular, \textit{ACDM} is unable to generate coherent samples even in a static denoising setting with 200 denoising steps. We did not explore \textit{ACDM} configurations with more than 200 diffusion steps due to the computational impracticality. The publicly available implementation of \textit{Thermalizer} fails to produce rollouts due to an internal safeguard that terminates execution when the inferred noise level exceeds a prescribed threshold.

We acknowledge that further fine-tuning of \textit{ACDM} and \textit{Thermalizer} can potentially result in models capable of simulating high-fidelity trajectories over long horizons. It is also worth noting that the same architecture used for the Kolmogorov flow did not succeed with the two-layer QG system at first, and a more expressive model architecture for the \textit{multi-noise-level score} estimator was needed. We provide more numerical details and varients of \textit{SNS} with ``easier'' optimization objective in  \cref{APPB}.   

Additional videos for visualization is provided \href{https://drive.google.com/drive/folders/1eUvQku5rJt_njWyM4YLRVInzGifT-qCj?usp=drive_link}{here}. And the code used for the numerical experiments in this paper is publicly available: \href{https://github.com/QiLiu6/SNS}{here}.

\section{Conclusion}
We introduced \textit{multi-noise-level denoising oracle}, a principled way to explicitly balance \textit{conditional approximation error} and \textit{OOD error}, formalizing heuristic-based methods for neural simulation over a long time horizon. Within this framework, we identified a central question: how to achieve an optimal tradeoff between \textit{conditional approximation error} and \textit{OOD error} through controlled perturbations of the conditional variable. To address this question, we proposed the \textit{Self-refining Neural Surrogate} model (\textit{SNS}), a conditional diffusion model which adaptively infers the optimal noise level of the conditional variable to preserve stability while respecting temporal dynamics.

Despite its conceptual succinctness, the current implementation of \textit{SNS} is restricted by a more challenging training objective. Regardless, we demonstrated competitive performance with models designed specifically for stability guarantees in long-time simulation of the Kolmogorov flow. We also demonstrated success of \textit{SNS} in modeling more ``complicated'' dynamics such as two-layer QG flows. We view \textit{SNS} as an initial step toward a broader class of methods enabled by the multi-noise-level denoising oracle, aimed at reconciling accurate short-term prediction with long-term distributional consistency. An important direction for future work is to better understand the interplay between the temporal correlations inherent in the underlying dynamical system and the model-induced distribution drift, which we believe is critical for achieving truly long-horizon neural simulations.

\section{Acknowledgment}
The authors would like to thank Chris Pedersen, Sara Vargas, Matthieu Blanke, Fabrizio Falasca, Pavel Perezhogin, Jiarong Wu, Oliver B\"uhler, Rudy Morel, Jonathan Weare, Sara Shamekh, and Ryan D\`u for many valuable discussions. This project is supported by Schmidt Sciences, LLC.

\section{Impact Statement}
This paper aims to help advance the field of Machine Learning. There are many potential societal consequences of our work, none of which we feel must be specifically highlighted here.
 

\bibliography{example_paper}
\bibliographystyle{icml2026}

\newpage

\appendix
\onecolumn
\section{Dynamical Systems}
We provide here an overview of the dynamical systems on which we validate our method. For a more thorough description, please refer to \cite{pedersen2025,ross2023}.
\subsection{Kolmogorov Flow}
The governing PDEs for the system are the 2D incompressible Navier-Stokes equations with sinusoidal forcing: 
\begin{equation*}
\begin{aligned}
    \partial_t \mathbf{u}+ (\mathbf{u}\cdot \nabla)\mathbf{u}&=\nu\nabla^2\mathbf{u}-\frac{1}{\rho}\nabla p+\mathbf{f} \\
    \nabla \cdot \mathbf{u} &= 0
\end{aligned}
\end{equation*}
where $\mathbf{u} = (u,v)$ is the two dimensional velocity field, $\nu$ is the kinematic viscosity, $\rho$ is the fluid density, $p$ is the pressure, and $\mathbf{f}$ is an external forcing term. We choose the constant sinusoidal forcing $\mathbf{f} = (\sin(4y), 0)$ as done in \cite{Kochkov2021}. Following \cite{pedersen2025}, we set $p=1, \nu = 0.001$ which corresponds to a Reynolds number $\mathrm{Re} = 10,000$. Defining the vorticity as the 2D curl of the velocity field, i.e. $\omega := \nabla_{H} \times \mathbf{u}$. One can obtain the vorticity equation by taking the curl of the above equations:
\begin{equation*}
\partial_t \omega + \mathbf{u}\cdot\nabla \omega
= \nu \nabla^2 \omega + (\nabla \times \mathbf{f})\cdot \hat{\mathbf{z}} .
\end{equation*}
The equations are solved using the pseudo-spectral method with periodic boundary conditions from the open-source code jax-cfd \cite{dresdner2023}. All results presented in this paper are based on the vorticity fields. 

\subsection{Two Layer Quasigeostrophic Equations}
For a two-layer fluid, defining the potential vorticity to be:
\begin{equation}
q_i=\nabla^2\psi_i+(-1)^i\frac{f_0^2}{g'H_m}(\psi_1-\psi_2),m\in\{1,2\},
\end{equation}

where $i=1$ and $i=2$ denote the upper layer and the lower layer, respectively. $H_i$ is the average depth of the layer, $\psi$ is the streamfunction, and $\mathbf{u}_i = \nabla^\perp \psi_i$. $f_0$ is the Coriolis frequency. The time evolution of the system is given by
\begin{equation}
\partial_tq_i+J(\psi_i, q_i)+\beta_i\partial_x\psi_i+U_i\partial_xq_i=\\
-\delta_{i,2}r\nabla^2\psi_i+\mathrm{ssd}
\label{eq:q_dt}
\end{equation}
where $U_i$ is the mean flow in $x$, $\beta_i$ accounts for the beta effect, $r$ is the bottom drag coefficient, and $\delta_{i,2}$ is the Kronecker delta. We set the parameters following the jet configuration in \cite{ross2023} and solve the equations with pyqg with numerical timestep $\delta t = 3600.0$.\cite{ryanabernathey} The physical time step in \cref{fig:QG} are subsampled by a factor of ten so the physical time step is $\Delta t = 36000.0$.
\begin{figure*}
    \centering
    \includegraphics[width=0.4\linewidth]{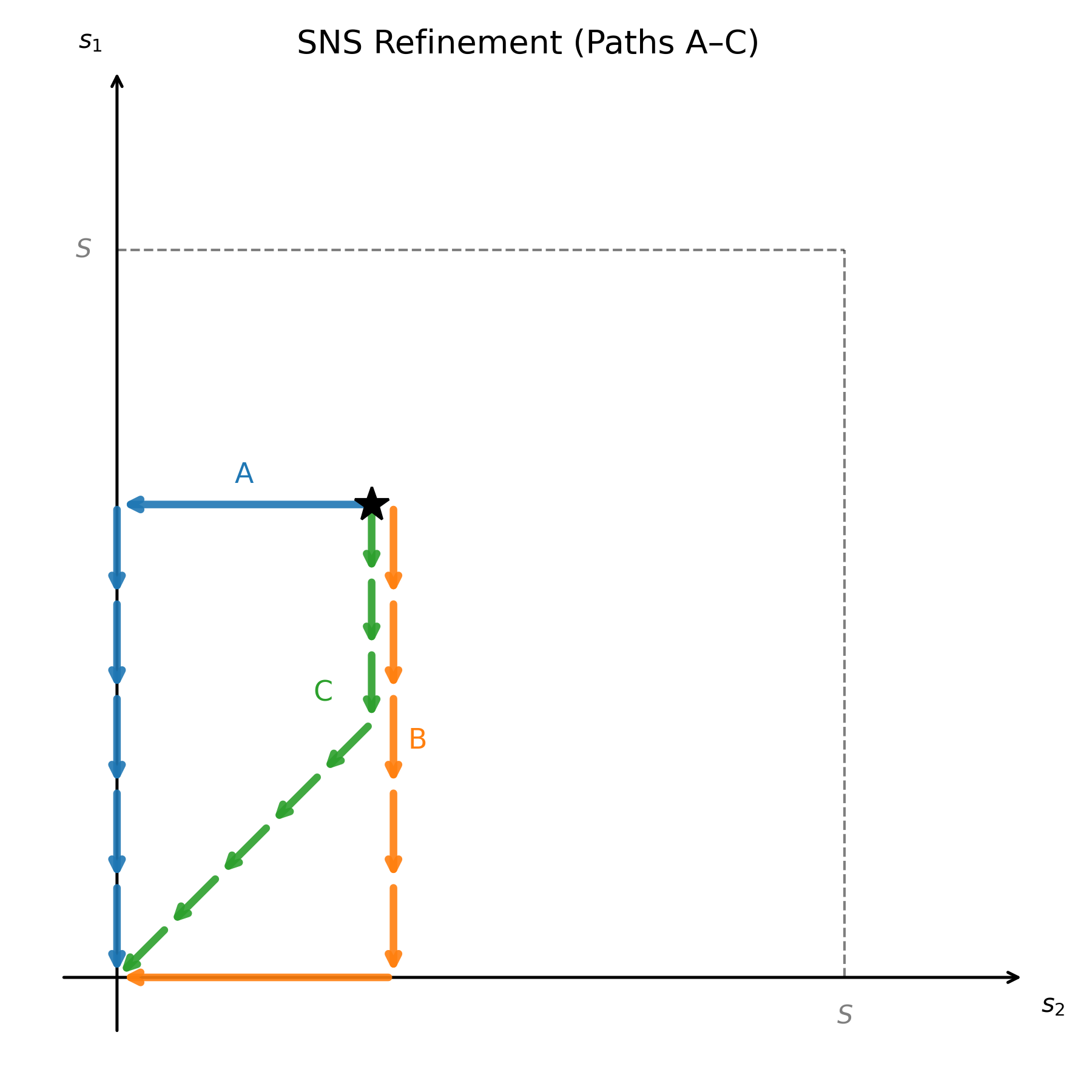}
    \caption{Different reasonable traversal strategies in the reverse SDE phase space}
    \label{fig:traverse}
\end{figure*}
\section{Numerical details}
\label{APPB}
\subsection{Model Architecture}
All models are based on a U-Net style encoder--decoder architecture \cite{ronneberger2015} implemented as a residual U-Net with multi-scale skip connections. For the \textit{multi-noise-level score} estimator, two auxiliary regression heads are attached to the bottleneck of the Unet as done by \cite{pedersen2025}. 

The network first projects the inputs to a 64-channel feature map using a $3\times3$ convolution with circular padding, followed by GELU activations throughout the network. At the bottleneck, the model applies two residual blocks operating on 512-channel feature maps. The decoder mirrors the encoder structure, using transposed convolutions ($4\times4$, stride 2) for upsampling and concatenating the corresponding encoder features via skip connections. Each decoder stage applies a residual block to fuse the upsampled and skip-connected features, progressively reducing the number of channels back to 64. In addition, the model used for the \textit{multi-noise-level score} estimator of the QG flows has attention heads attached to the bottleneck layer. The model did not work well without this addition.

\subsection{Training}
\label{Training}
The training of the Gaussian approximation, $p_\theta$, of the transition density with point-wise estimate mean and fixed variance is done via minimizing the multi-step mean squared error on the residuals of the snapshots, i.e.
\begin{equation*}
    \mathcal{L}(\theta) =\frac{1}{N}\sum_{k=1}^N\sum_{t=0}^{L-1}\parallel\Psi_\theta(\mathbf{x}_t^k)-(\mathbf{x}_{t+1}^k-\mathbf{x}_{t}^k)+\sigma \mathbf{z}\parallel^2.
\end{equation*}

where $\mathbf{z}\in \mathcal{N}(0,I)$, $p_\theta = \mathcal{N}(\Psi_\theta,\sigma^2I)$. We follow the training configuration proposed by \cite{pedersen2025} to obtain a similar baseline model. 

We train \textit{SNS} using the \textit{DDPM} style objective with 440,000 training pair $\mathbf{y}_t = (\mathbf{x}_{t},\mathbf{x}_{t-1})$.
We choose $\{ \beta^s \}_{s=1}^S$ to be the cosine variance schedule as proposed in \cite{Ho2020} with $\alpha^s = 1 - \beta^s$ and
$\bar{\alpha}^s = \prod_{k=1}^s \alpha^k$.
The forward diffusion process has the transition density:
\begin{equation*}
q(x^s \mid x^0) = \mathcal{N}\!\left(
 \sqrt{\bar{\alpha}^s}\, x^0,\; (1 - \bar{\alpha}^s) I
\right),
\end{equation*}
which allows sampling $x^s_t$ in closed form as
\begin{equation*}
x_t^{s_1} = \sqrt{\bar{\alpha}^{s_1}}\, x_t^0 + \sqrt{1 - \bar{\alpha}^{s_1}}\, \epsilon,
\quad x_{t-1}^{s_2} = \sqrt{\bar{\alpha}^{s_2}}\, x_{t-1}^0 + \sqrt{1 - \bar{\alpha}^{s_2}}\, \epsilon',
\qquad \epsilon, \epsilon' \sim \mathcal{N}(0, I).
\end{equation*}
We noise the inputs to the neural net independently with $s_1, s_2 \sim \mathcal{U}(1,S)$, and the loss function we minimize is:
\begin{equation*}
   \mathcal{L}(\theta) := \mathbb{E}\{\| (\epsilon,\epsilon') - \mathbf{\Phi}^{(3)}_\theta(\mathbf{y}^{\mathbf{s}}_t) \|^2 +\lambda\sum_{s_1,s_2}[ \mathbf{1}_{s_1}\log \mathbf{\Phi}^{(1)}_\theta (\mathbf{y}^{\mathbf{s}}_t) + \mathbf{1}_{s_2}\log \mathbf{\Phi}^{(2)}_\theta (\mathbf{y}^{\mathbf{s}}_t)]
\end{equation*}
where the terms were defined before \cref{loss}.
The minimizer gives us access to the \textit{multi-noise-level score} with:
\begin{equation*}
\epsilon
=
-\sqrt{1-\bar{\alpha}_t}\;
\nabla_{x_t} \log q(x^s \mid x^0).
\end{equation*}

We use AdamW with a learning rate of $5e-4$ and set the weighting parameter $\lambda = 0.1$. Each training run requires approximately two days on a single NVIDIA H200 GPU, using a U-Net with 73,727,522 trainable parameters. Empirically, we find that augmenting the input with a sample from the invariant measure improves convergence, even though the trained model ultimately exhibits negligible dependence on this signal, indicating that its benefit is primarily during the early stages of training.

\subsection{\textit{SNS} Traversal Strategies}
\label{traverse}
Mathematically, all continuous transversal strategies that end at the point $(0,0)$ following a monotonically decreasing path satisfy the reverse process SDE. Thus, it appears to be that there should be no difference between different choices of traversal. \cref{fig:traverse} shows some reasonable traversal strategies in the reverse process phase space. Through numerical experiments, the most robust strategy appears to be path A in \cref{fig:traverse}. Despite numerical impracticality, we still present \cref{alg:gen} for using $SNS$ to directly generate next state prediction. 

We also provide a histogram of the denoising steps associated with the two numerical experiments presented in the main text in \cref{fig:histogram}. It is worth noting that, in practice, one could choose an activation threshold for the algorithms presented, i.e. \textit{SNS} only starts denoising when a certain noise level is reached. This can improve the inference time of \textit{SNS} while preserving most of its performance due to the extremely high signal-to-noise ratio at small noise levels.
\begin{figure*}
    \centering
    \includegraphics[width=\linewidth]{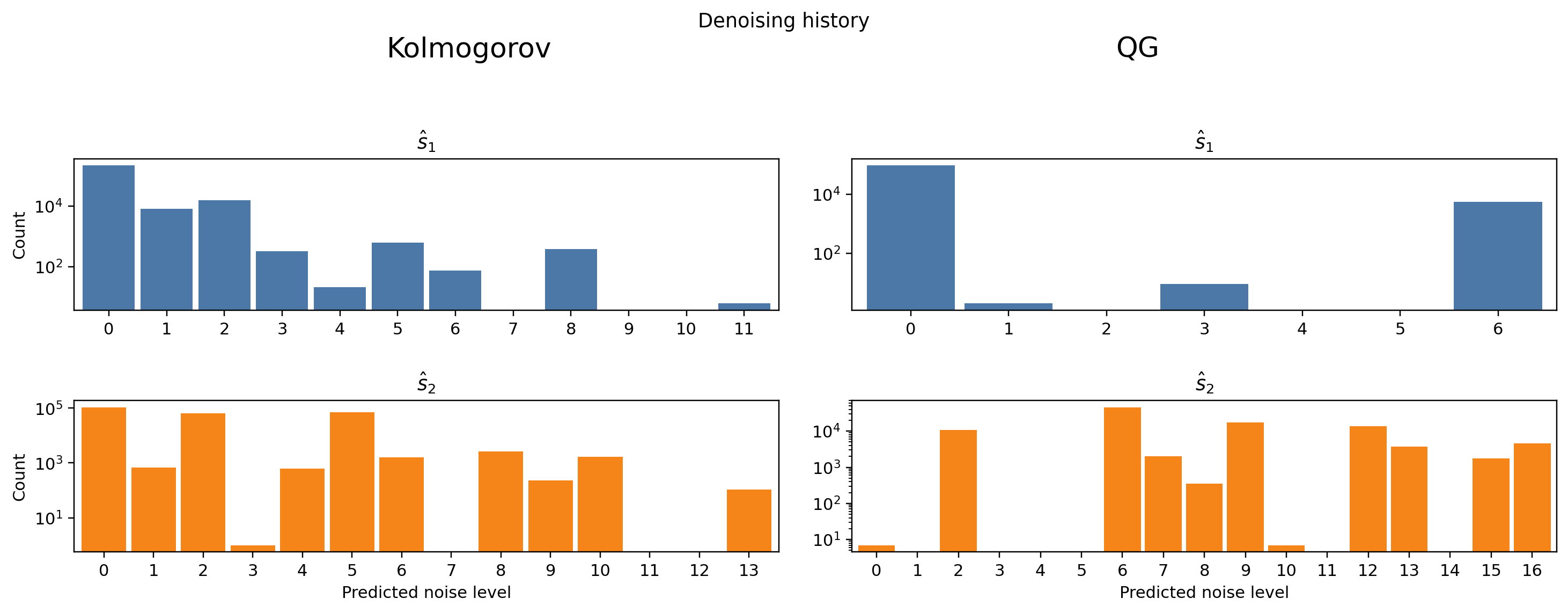}
    \caption{Histogram of denoising history for the numerical experiments in the main text}
    \label{fig:histogram}
\end{figure*}

\begin{algorithm}[tb]
  \caption{\textit{SNS} for generation}
  \label{alg:gen}
  \begin{algorithmic}
    \REQUIRE Initial state $\mathbf{x}_0$, trained \textit{SNS} $\mathbf{\Phi}_{\theta}$
    \FOR{$t=1$ {\bfseries to} $T$}
    \STATE $\hat{\mathbf{x}}_t, \mathbf{z} \sim\mathcal{N}(\mathbf{0},\mathbf{I})$
    \STATE $\hat{s}_2 = \mathbf{\Phi}^{(2)}(\hat{\mathbf{x}}_t,\hat{\mathbf{x}}_{t-1})$
    \STATE $\mathbf{x}_{t-1}^{\hat{s}_2} = \sqrt{\bar{\alpha}_{\hat{s}_s}}\mathbf{x}_{t-1} + \sqrt{1-\bar{\alpha}_{\hat{s}_s}}\mathbf{z}$
    \FOR{$s=S$ {\bfseries to} $0$}
    \STATE $\mathbf{z}^s\sim\mathcal{N}(\mathbf{0},\mathbf{I})$
    \STATE $\hat{\mathbf{x}}_t^s = \frac{1}{\sqrt{\alpha_s}}\left(\hat{\mathbf{x}}_t^s - \frac{1-\alpha_s}{\sqrt{1-\bar{\alpha}}_s} \mathbf{\Phi}^{(3)}_\theta(\hat{\mathbf{x}}_t^s,\hat{\mathbf{x}}_{t-1}^{s_2}) \right) + \sqrt{\beta_s} \mathbf{z}^s$
    \ENDFOR
    \ENDFOR
    \end{algorithmic}
\end{algorithm}

\subsection{\textit{SNS Variants}}
\label{Var}

We consider conditioning the \textit{multi-noise-level denoising oracle}
$\mathbf{D}(\mathbf{x}_t^{s_1}, \mathbf{x}_{t-1}^{s_2})$ on a snapshot from the past,
$\mathbf{x}_{t-r}$. When $r \gg 1$, the conditional denoising oracle recovers the
unconditional oracle due to the chaotic nature of the underlying system. Similarly,
in case when $r = 1$, the conditional denoising oracle also reduces to
the unconditional case as a consequence of the Markovian property of the dynamics.

The regime of interest is when $\mathbf{x}_{t-1}$ is perturbed by noise and $r$ is
chosen such that $\mathbf{x}_{t-r}$ remains partially coupled to $\mathbf{x}_{t-1}$,
but is not trivially redundant. In this setting, the parameter $r$ can be interpreted
as controlling the error trade-off. Analogous to the SNS procedure for
estimating the noise level of $\mathbf{x}_{t-1}^{s_2}$, one could in principle estimate how far back in the trajectory one must condition in order to sufficiently weaken temporal dependence while maintaining numerical stability.

This alternative approach is considerably more challenging to implement, as it
requires access to historical trajectory data. In practice, storage of long trajectories as training data may already be computationally demanding. Nevertheless,
we empirically observe that providing a noise-free past frame at a predetermined lag, $r$, as an auxiliary conditioning signal substantially simplifies the optimization problem.

The key challenge lies in selecting a conditioning frame that is neither too close
to the present (leading to excessive dependence) nor too distant (resulting in
insufficiently informative conditioning). We keep this discussion informal, as the
appropriate choice of $r$ is heuristic and system-dependent. In our experiments,
conditioning on a clean frame 100 steps in the past consistently improved optimization performance and was effective for stably simulating the two layer quasigeostrophic system using the same architecture backbone used for the Kolmogorov flow. (\cref{fig:QG})

\section{Proofs}
\label{AppC}

\begin{theorem}[Equivalence between monotonic traversal strategies]
\label{tra_proof}

Starting from an initial point 
$(\mathbf{x}_t^{s_1^0},\mathbf{x}_{t-1}^{s_2^0})$, consider a discretized
traversal path in reverse process phase space : $\{(s_1^k,s_2^k)\}_{k=0}^{N}$,  with $(s_1^0,s_2^0)$ the starting point and $(s_1^N,s_2^N)=(0,0)$.
Assume the path is monotonically decreasing in each coordinate: $s_i^{k+1} \le s_i^{k}, ~ i=1,2.$
Let $(\mathbf{x}_t^{s_1^k},\mathbf{x}_{t-1}^{s_2^k})$ evolve according to the
coupled reverse SDE defined by the multi-noise-level score.
Then the resulting joint distribution of the initial and terminal state
\[
p\big(
\mathbf{x}_t^{s_1^N},\mathbf{x}_{t-1}^{s_2^N},
\mathbf{x}_t^{s_1^0},\mathbf{x}_{t-1}^{s_2^0}
\big)
\]
is independent of the particular traversal path in reverse process phase space, i.e for any two monotonically decreasing paths
\[
\{(s_1^k,s_2^k)\}_{k=0}^{N},
\qquad
\{(s_1^{\prime k},s_2^{\prime k})\}_{k=0}^{N'},
\]
connecting $(s_1^0,s_2^0)$ to $(s_1^N,s_2^N)$, the random variables generated by the reverse coupled SDE satisfy
\[
p\!\left(
\mathbf{x}_t^{s_1^N},\mathbf{x}_{t-1}^{s_2^N} , 
\mathbf{x}_t^{s_1^0},\mathbf{x}_{t-1}^{s_2^0}
\right)
=
p\!\left(
\mathbf{x}_t^{s_1^{'N}},\mathbf{x}_{t-1}^{s_2^{'N}},
\mathbf{x}_t^{s_1^0},\mathbf{x}_{t-1}^{s_2^0}
\right),
\]

\begin{proof}
Fix a monotonically decreasing discrete path in reverse process phase space
\[
\{(s_1^k,s_2^k)\}_{k=0}^{N}, 
\qquad (s_1^0,s_2^0)\to (s_1^N,s_2^N)=(0,0),
\qquad s_i^{k+1}\le s_i^k.
\]
Since the path is monotonically decreasing, we can parametrize this path with $u\in[0,1]$ and define
\[
u \mapsto \mathbf{s}(u) \coloneqq (s_1(u),s_2(u)),
\qquad s_i(\cdot)\ \text{nonincreasing},\quad 
\mathbf{s}(0)=(s_1^0,s_2^0),\ \mathbf{s}(1)=(0,0),
\]
such that the discretization points are contained in the image
$\{\mathbf{s}(u)\,:\,u\in[0,1]\}$ (e.g.\ take $s_i(u)$ piecewise linear through
$\{s_i^k\}$). Define the coupled state along the path by the single
path-parameterized random variable
\[
\mathbf{Y}_u \;\coloneqq\; 
\big(\mathbf{x}_t^{s_1(u)},\mathbf{x}_{t-1}^{s_2(u)}\big).
\]
then the increments
$ds_1$ and $ds_2$ along the path can be written as
\[
ds_1 = \dot s_1(u)\,du,\qquad ds_2=\dot s_2(u)\,du,
\qquad \dot s_i(u)\le 0\
\]
Substituting these relations into the coupled reverse SDE yields a single
SDE in the parameter $u$ for $\mathbf{Y}_u$ (interpreting the stochastic
integrals componentwise with respect to the time-changed Brownian motions):
\begin{equation*}
\label{eq:param-reverse}
\begin{aligned}
d\mathbf{x}_t^{s_1(u)}
=&\Big[\mathbf{f}(\mathbf{x}_t^{s_1(u)},s_1(u))
- g(s_1(u))^2\, \nabla_{\mathbf{x}_t}\log p^{\mathbf{s}(u)}(\mathbf{y}_t^{\mathbf{s}(u)})\Big]\,
\dot s_1(u)\,du
\\
&\quad + g(s_1(u))\, d\bar{\mathbf{w}}^{s_1(u)},
\\
d\mathbf{x}_{t-1}^{s_2(u)}
=&\Big[\mathbf{f}(\mathbf{x}_{t-1}^{s_2(u)},s_2(u))
- g(s_2(u))^2\, \nabla_{\mathbf{x}_{t-1}}\log p^{\mathbf{s}(u)}(\mathbf{y}_t^{\mathbf{s}(u)})\Big]\,
\dot s_2(u)\,du
\\
&\quad + g(s_2(u))\, d\bar{\mathbf{w}}^{s_2(u)}.
\end{aligned}
\end{equation*}
Thus, following a monotone path reduces the evolution to the single random
variable $\mathbf{Y}_u$ driven by the parameter $u$.

\medskip

Now consider the forward process applied independently to each
component,
\[
d\mathbf{x}^s = \mathbf{f}(\mathbf{x}^s,s)\,ds + g(s)\,d\mathbf{w}^s,
\qquad s\in[0,S],
\]
and run it along the same monotone parametrization $u\mapsto(s_1(u),s_2(u))$,
i.e.\ consider the forward path process
\[
\mathbf{Y}_u^{\mathrm{fwd}}
\;\coloneqq\;
\big(\mathbf{x}_t^{s_1(u)},\mathbf{x}_{t-1}^{s_2(u)}\big)
\quad\text{with}\quad
d\mathbf{x}_t^{s_1(u)} = \mathbf{f}(\mathbf{x}_t^{s_1(u)},s_1(u))\,ds_1(u) + g(s_1(u))\,d\mathbf{w}^{s_1(u)},
\]
and analogously for $\mathbf{x}_{t-1}^{s_2(u)}$.
By construction, the reverse coupled SDE uses the multi-noise-level score
$\nabla \log p^{\mathbf{s}}(\mathbf{y}_t^{\mathbf{s}})$, which is precisely
the object that defines the time-reversal of the forward marginals at each
noise-level $\mathbf{s}=(s_1,s_2)$. Consequently, when the reverse SDE is initialized at $\mathbf{Y}_1=(\mathbf{x}_t^{s_1^N},\mathbf{x}_{t-1}^{s_2^N})$, its induced
law on path space $\{\mathbf{Y}_u: u\in[0,1]\}$ matches (in reverse time) the
law of the forward process runs along the same path. In particular, the
reverse process produces the correct conditional endpoint distribution
associated with that path:
\begin{equation*}
\label{eq:endpoint-match}
p\!\left(\mathbf{Y}_0^{bwd}, \mathbf{Y}_1^{bwd}\right)
=
p\!\left(\mathbf{Y}_0^{fwd}, \mathbf{Y}_1^{fwd}\right),
\end{equation*}

\medskip

Now take any two monotone paths $\mathbf{s}(u)=(s_1(u),s_2(u))$ and
$\mathbf{s}'(u)=(s_1'(u),s_2'(u))$ connecting the same endpoints
$(s_1^0,s_2^0)$ and $(0,0)$. Under the forward dynamics, the two coordinates
evolve independently given their own noise levels, i.e.\ the law of
$\mathbf{x}_t^{s_1}$ depends only on $s_1$ and the law of $\mathbf{x}_{t-1}^{s_2}$
depends only on $s_2$, with independent driving Brownian motions. Therefore,
for any such path, the forward endpoint distribution at $(0,0)$ factors as
\[
p\!\left(\mathbf{x}_t^{s_1^0},\mathbf{x}_{t-1}^{s_2^0},
\mathbf{x}_t^{0},\mathbf{x}_{t-1}^{0}\right)
=
\left(p\!\left(\mathbf{x}_t^{s_1^0}\mid\mathbf{x}_t^{0}\right)
,p\!\left(\mathbf{x}_{t-1}^{s_2^0}\mid\mathbf{x}_{t-1}^{0}\right)\right) \otimes p(\mathbf{x}_t^{0},\mathbf{x}_{t-1}^{0}),
\]
where $\otimes$ is element-wise multiplication. The factoring clearly does not depend on how one jointly traverses $(s_1,s_2)$, only on the endpoints in each coordinate. 

Thus, we conclude that the reverse process generates the right distribution independent of the traversal path.
\end{proof}
\end{theorem}

\subsection*{Equivalence of minimizers: joint score vs.\ conditional scores}
Fix $\mathbf{s}=(s_1,s_2)\in[0,S]^2$ and let
$\mathbf{y}_t^{\mathbf{s}} := (\mathbf{x}_t^{s_1},\mathbf{x}_{t-1}^{s_2}) \sim p^{\mathbf{s}}$
denote the noised pair at noise levels $\mathbf{s}$.
Consider the denoising score matching objective
\begin{equation}
\label{eq:c1_joint_obj}
\mathcal{L}_{\mathrm{joint}}(\theta)
:=
\mathbb{E}_{\mathbf{s}\sim\mathcal{U}(0,S)^2}\,
\mathbb{E}_{\mathbf{y}_t^0\sim p}\,
\mathbb{E}_{\mathbf{y}_t^{\mathbf{s}}\sim p^{\mathbf{s}}(\cdot\mid \mathbf{y}_t^0)}
\Big\|
\nabla_{\mathbf{y}} \log q_{\mathbf{s}}(\mathbf{y}\mid \mathbf{y}_t^0)
-
\mathbf{\Phi}_\theta(\mathbf{y}_t^{\mathbf{s}},s_1,s_2)
\Big\|_2^2,
\end{equation}
where $\mathbf{y}=\mathbf{y}_t^{\mathbf{s}}$ and $p^{\mathbf{s}}(\cdot \mid \cdot)$ is the forward noising kernel.

Also consider the objective
\begin{equation}
\label{eq:c1_cond_obj}
\mathcal{L}_{\mathrm{cond}}(\theta)
:=
\mathbb{E}_{\mathbf{s}\sim\mathcal{U}(0,S)^2}\,
\mathbb{E}_{\mathbf{y}_t^{\mathbf{s}}\sim p^{\mathbf{s}}}
\Big\|
\begin{pmatrix}
\nabla_{\mathbf{x}_t} \log p^{s_1}\!\left(\mathbf{x}_t^{s_1}\mid \mathbf{x}_{t-1}^{s_2}\right)\\[2pt]
\nabla_{\mathbf{x}_{t-1}} \log p^{s_2}\!\left(\mathbf{x}_{t-1}^{s_2}\mid \mathbf{x}_t^{s_1}\right)
\end{pmatrix}
-
\mathbf{\Phi}_\theta(\mathbf{y}_t^{\mathbf{s}},s_1,s_2)
\Big\|_2^2.
\end{equation}

Then given a parametric class $\Theta$, the sets of minimizers coincide:
\[
\arg\min_{\theta\in\Theta}\,\mathcal{L}_{\mathrm{joint}}(\theta)
\;=\;
\arg\min_{\theta\in\Theta},\mathcal{L}_{\mathrm{cond}}(\theta),
\]

\paragraph{Notation.}
Let $p^{\mathbf{s}}(\mathbf{y}_t^{\mathbf{s}} \mid \mathbf{y}_t^0)$ denote
the forward noising kernel induced by the independent SDEs, and define
the marginal
\[
p^{\mathbf{s}}(\mathbf{y}_t^{\mathbf{s}})
:= \int p(\mathbf{y}_t^0)\,
p^{\mathbf{s}}(\mathbf{y}_t^{\mathbf{s}} \mid \mathbf{y}_t^0)\,
d\mathbf{y}_t^0 .
\]

\begin{lemma}
\label{lem:tower_score}
Fix $\mathbf{s}=(s_1,s_2)$. For any measurable
$\boldsymbol{\psi}(\mathbf{y})$,
the objectives
\begin{align*}
\mathbb{E}_{\mathbf{y}_t^0 \sim p^0}
\mathbb{E}_{\mathbf{y}_t^{\mathbf{s}} \sim p^{\mathbf{s}}(\cdot\mid\mathbf{y}_t^0)}
\Big[
\|
\nabla_{\mathbf{y}} \log p^{\mathbf{s}}(\mathbf{y}^\mathbf{s}\mid\mathbf{y}_t^0)
- \boldsymbol{\psi}(\mathbf{y})
\|^2
\Big]
\end{align*}
and
\begin{align*}
\mathbb{E}_{\mathbf{y} \sim p^{\mathbf{s}}}
\Big[
\|
\nabla_{\mathbf{y}} \log p^{\mathbf{s}}(\mathbf{y})
- \boldsymbol{\psi}(\mathbf{y})
\|^2
\Big]
\end{align*}
have the same minimizer over $\boldsymbol{\psi}$, given by
\[
\boldsymbol{\psi}^*(\mathbf{y})
=
\nabla_{\mathbf{y}} \log p^{\mathbf{s}}(\mathbf{y}^s).
\]
\end{lemma}

\begin{proof}
Fix $\mathbf{s}$ and abbreviate $\mathbf{y}=\mathbf{y}_t^{\mathbf{s}}$.
By the tower property of conditional expectation,
\[
\mathbb{E}_{\mathbf{y}_t^0}
\mathbb{E}_{\mathbf{y}\mid\mathbf{y}_t^0}[\cdot]
=
\mathbb{E}_{\mathbf{y}\sim p^{\mathbf{s}}}
\mathbb{E}_{\mathbf{y}_t^0\mid\mathbf{y}}[\cdot].
\]
Hence the first objective can be written as
\[
\mathbb{E}_{\mathbf{y}\sim p^{\mathbf{s}}}
\mathbb{E}_{\mathbf{y}_t^0\mid\mathbf{y}}
\Big[
\|
\nabla_{\mathbf{y}} \log p^{\mathbf{s}}(\mathbf{y}\mid\mathbf{y}_t^0)
- \boldsymbol{\psi}(\mathbf{y})
\|^2
\Big].
\]
For each fixed $\mathbf{y}$, the minimizer of the inner expectation is the
conditional mean,
\[
\boldsymbol{\psi}^*(\mathbf{y})
=
\mathbb{E}_{\mathbf{y}_t^0\mid\mathbf{y}}
\big[
\nabla_{\mathbf{y}} \log p^{\mathbf{s}}(\mathbf{y}\mid\mathbf{y}_t^0)
\big].
\]
Using the identity $\nabla p = p \nabla \log p$, we compute
\begin{align*}
\nabla_{\mathbf{y}} p^{\mathbf{s}}(\mathbf{y})
&= \nabla_{\mathbf{y}}
\int p(\mathbf{y}_t^0)\, q^{\mathbf{s}}(\mathbf{y}\mid\mathbf{y}_t^0)\,
d\mathbf{y}_t^0 \\
&= \int p(\mathbf{y}_t^0)\,
p^{\mathbf{s}}(\mathbf{y}\mid\mathbf{y}_t^0)\,
\nabla_{\mathbf{y}} \log p^{\mathbf{s}}(\mathbf{y}\mid\mathbf{y}_t^0)\,
d\mathbf{y}_t^0 .
\end{align*}
Since
$p(\mathbf{y}_t^0)\,p^{\mathbf{s}}(\mathbf{y}\mid\mathbf{y}_t^0)
= p^{\mathbf{s}}(\mathbf{y})\,p(\mathbf{y}_t^0\mid\mathbf{y})$,
we obtain
\[
\nabla_{\mathbf{y}} \log p^{\mathbf{s}}(\mathbf{y})
=
\mathbb{E}_{\mathbf{y}_t^0\mid\mathbf{y}}
\big[
\nabla_{\mathbf{y}} \log p^{\mathbf{s}}(\mathbf{y}\mid\mathbf{y}_t^0)
\big],
\]
which proves the claim.
\end{proof}

\begin{theorem}[Minimizer of the multi-noise-level DSM objective]
\label{thm:dsm_minimizer}
The minimizer of the objective \eqref{eq3} is given (for almost every
$\mathbf{s}$) by
\[
\mathbf{\Phi}^*(\mathbf{y}_t^{\mathbf{s}},\mathbf{s})
=
\nabla_{\mathbf{y}} \log p^{\mathbf{s}}(\mathbf{y}_t^{\mathbf{s}}).
\]
\end{theorem}

\begin{figure*}
    \centering
    \includegraphics[width=1\linewidth]{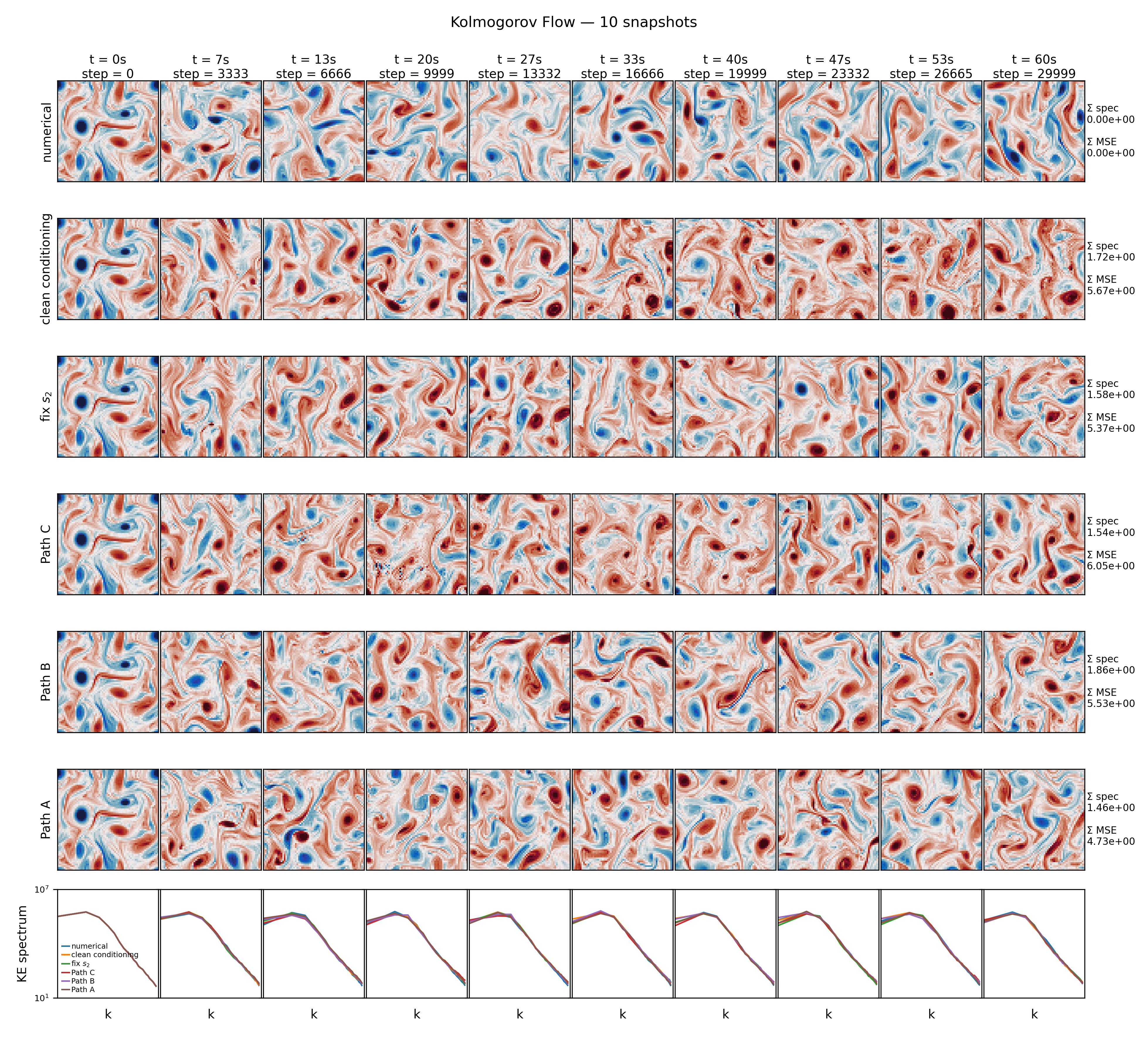}
    \caption{Each row in a rollout from the same initial condition using point-wise estimate surrogate and \textit{SNS} refinement, but each following a different \textit{traversal strategy}. The spectra are averaged over 20 snapshots at time t with independent initial conditions. On the right is the $L^2$ distance between the temporal mean of numerical fields and the mean of each run, and similarly for the $L^2$ distance between the densities. }
    \label{fig:traverse}
\end{figure*}

\begin{proof}
The objective \eqref{eq3} is an expectation over
$\mathbf{s}\sim\mathcal U(0,S)^2$ of the objectives considered in
Lemma~\ref{lem:tower_score}. Since each $\mathbf{s}$-slice is minimized by
$\nabla_{\mathbf{y}} \log p_{\mathbf{s}}(\mathbf{y})$, the full objective
is minimized by the same function almost everywhere in $\mathbf{s}$.
\end{proof}

\begin{corollary}[Equivalence with conditional scores]
\label{cor:conditional_scores}
For $\mathbf{y}=(\mathbf{x}_t^{s_1},\mathbf{x}_{t-1}^{s_2})$,
\[
\nabla_{\mathbf{y}} \log p^{\mathbf{s}}(\mathbf{y})
=
\big(
\nabla_{\mathbf{x}_t} \log p^{s_1}(\mathbf{x}_t^{s_1}\mid\mathbf{x}_{t-1}^{s_2}),
\;
\nabla_{\mathbf{x}_{t-1}} \log p^{s_2}(\mathbf{x}_{t-1}^{s_2}\mid\mathbf{x}_t^{s_1})
\big).
\]
Consequently, replacing the \textit{multi-noise-level score} in \eqref{eq3} by the
pair of conditional scores yields an objective with the same minimizer.
\end{corollary}

\begin{proof}
By the factorization
\[
\log p_{\mathbf{s}}(\mathbf{x}_t^{s_1},\mathbf{x}_{t-1}^{s_2})
=
\log p^{s_1}(\mathbf{x}_t^{s_1}\mid\mathbf{x}_{t-1}^{s_2})
+
\log p^{s_2}(\mathbf{x}_{t-1}^{s_2}),
\]
the second term does not depend on $\mathbf{x}_t^{s_1}$, yielding the first
identity. The second follows analogously.
\end{proof}

\begin{corollary}
    Minimizer of conditional approximation error requires clean conditioning.
\end{corollary}
Suppose $d$ is the conditional KL divergence:
\begin{equation}
\label{eq:cond_kl_risk}
\mathcal{R}_{\mathrm{cond}}(q)
:=
\mathbb{E}_{x\sim \mu}\Big[
\mathrm{KL}\big(p(\cdot\mid x)\,\|\,q(\cdot\mid x)\big)
\Big],
\end{equation}
where $x=x_{t-1}$ and $p(\cdot\mid x)=p(x_t\mid x_{t-1})$.
It is immediate that the unique minimizer (a.e.\ in $x$) is $q^*(\cdot\mid x)=p(\cdot\mid x)$, with minimum value $0$.

Now suppose we do \emph{not} condition on the clean state $x$, but only on a
corrupted observation $\tilde{x}\sim r(\tilde{x}\mid x)$ (e.g.\ Gaussian
noising). Any predictor based on the corrupted input has the form
$q(x_t\mid \tilde{x})$. The best achievable predictor in this restricted class
is:
\begin{equation}
\label{eq:posterior_predictive}
q^*(x_t\mid \tilde{x}) \;=\; p(x_t\mid \tilde{x})
\;:=\; \int p(x_t\mid x)\,p(x\mid \tilde{x})\,dx.
\end{equation}

Moreover, the \emph{irreducible gap} between conditioning on $x$ and on
$\tilde{x}$ is exactly a conditional mutual information:
\begin{equation}
\label{eq:gap_cmi}
\inf_{q(\cdot\mid \tilde{x})}\;
\mathbb{E}\Big[
\mathrm{KL}\big(p(\cdot\mid x)\,\|\,q(\cdot\mid \tilde{x})\big)
\Big]
\;=\;
\mathbb{E}\Big[
\mathrm{KL}\big(p(\cdot\mid x)\,\|\,p(\cdot\mid \tilde{x})\big)
\Big]
\;=\;
I(X_t;X_{t-1}\mid \tilde{X}_{t-1}),
\end{equation}
where the expectation is under the joint law
$p(x_{t-1})\,p(x_t\mid x_{t-1})\,r(\tilde{x}_{t-1}\mid x_{t-1})$.

In particular, if the corruption is non-degenerate, then typically
$I(X_t;X_{t-1}\mid \tilde{X}_{t-1})>0$, so conditioning on $\tilde{x}_{t-1}$
cannot attain the minimum of \eqref{eq:cond_kl_risk}. Equality holds if and only
if $\tilde{X}_{t-1}$ is a sufficient statistic for $X_{t-1}$ with respect to
$X_t$, i.e.\ $X_t \perp X_{t-1}\mid \tilde{X}_{t-1}$; the degenerate case
$\tilde{X}_{t-1}=X_{t-1}$ recovers the clean-conditioning optimum.

\begin{figure*}
    \centering
    \includegraphics[width=0.95\linewidth]{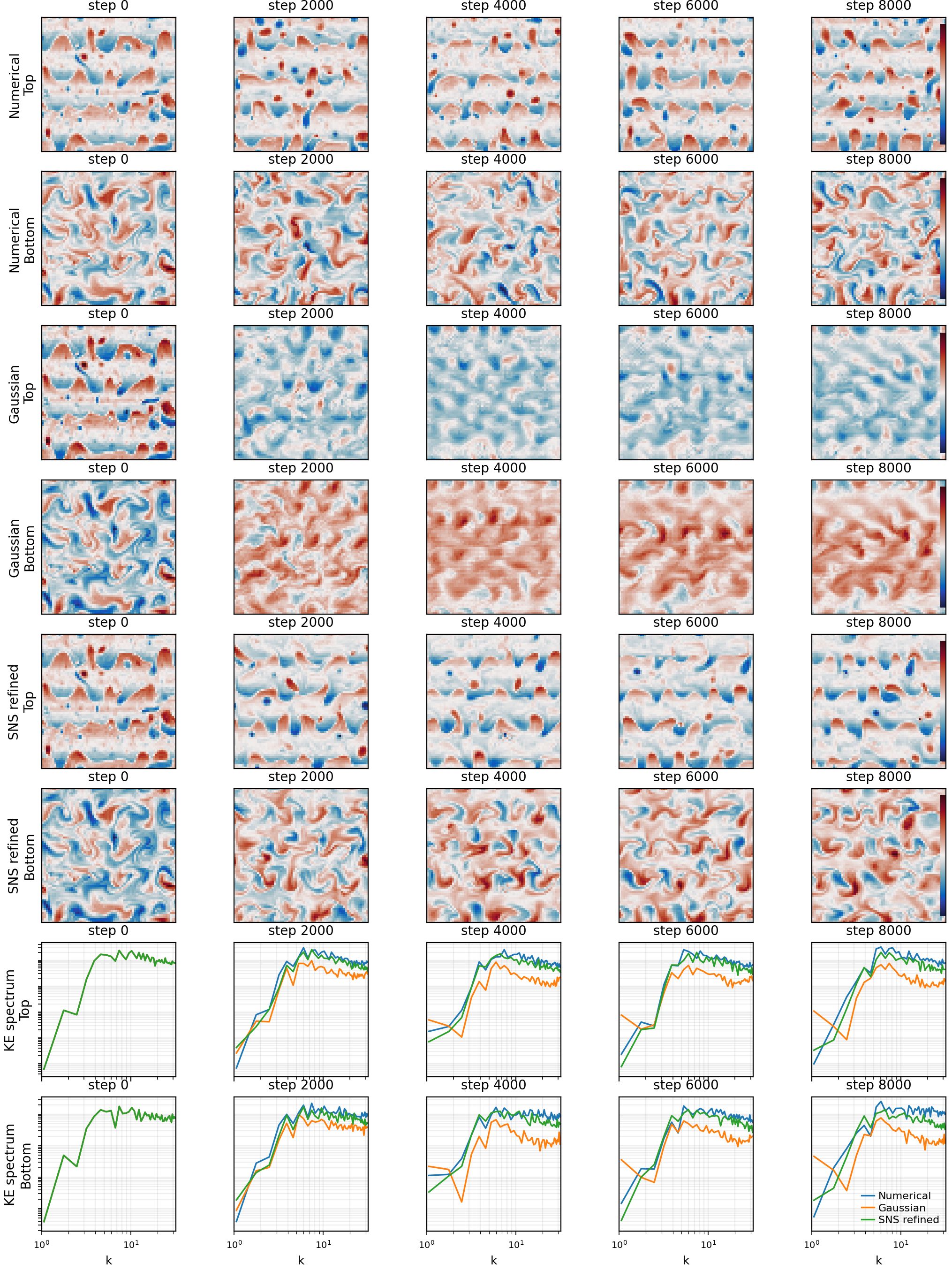}
    \caption{2 layer Quasigeostrophic Turbulence with \textit{SNS} varients}
    \label{fig:QG}
\end{figure*}


\end{document}